\documentclass[oneside,11pt]{article}

\usepackage{amsmath,amsthm,amsopn,amsfonts,pdfpages,url,bbm,amssymb,pdfpages}
\usepackage{authblk,dsfont}
\usepackage{graphics,relsize}
\usepackage{subfig}

\newtheorem{theorem}{Theorem}
\newtheorem{definition}[theorem]{Definition}
\newtheorem{proposition}[theorem]{Proposition}

\newtheorem{corollary}[theorem]{Corollary}
\newtheorem{remark}[theorem]{Remark}
\newtheorem{example}[theorem]{Example}

\newcommand{\mo}{\mathfrak{o}}
\newcommand{\am}{\text{argmax}}
\newcommand{\p}{\mathbb{P}}

\newcommand{\R}{\mathbb{R}}

\newcommand{\rank}{\operatorname{rank}}

\newcommand{\SBM}{\operatorname{SBM}}
\newcommand{\asymSBM}{\operatorname{asym-SBM}}
\newcommand{\RDPG}{\operatorname{RDPG}}
\newcommand{\ER}{\operatorname{ER}}
\newcommand{\RER}{R\operatorname{-ER}}

\newcommand{\calG}{\mathcal{G}}

\newcommand{\calT}{\mathcal{T}}

\newcommand{\calV}{\mathcal{V}}
\newcommand{\calX}{\mathcal{X}}
\newcommand{\bp}{\boldsymbol{\Phi}}

\newcommand{\calVnm}{\calV_{n,m}}
\newcommand{\ErdosRenyi}{\text{Erd\H{o}s-R\'{e}nyi }}

\newcommand{\Phistar}{\Phi^*}
\newcommand{\Lstar}{L^*}

\newcommand{\vstar}{v^*}

\newcommand{\calGn}{\calG_n}
\newcommand{\calGm}{\calG_m}
\newcommand{\gn}{\calGn}
\newcommand{\gm}{\calGm}

\newcommand{\BayesVN}{\Phistar}
\newcommand{\PFtau}{\mathbb{P}_{(G_1,G_2)\sim F_{c,n,m,\theta}}}
\newcommand{\PFt}{\mathbb{P}_{ F_{c,n,m,\theta}}}

\begin{document}

\title{On consistent vertex nomination schemes}

\author[$\dag$]{Vince Lyzinski}
\author[$\ddag$]{Keith Levin}
\author[$^*$]{Carey E. Priebe}

\affil[$\dag$]{\small Department of Mathematics and Statistics, University of Massachusetts Amherst}
\affil[$\ddag$]{\small Department of Statistics, University of Michigan}
\affil[$^*$]{\small Department of Applied Mathematics and Statistics, Johns
    Hopkins University}

\maketitle

\begin{abstract}
Given a vertex of interest in a network $G_1$, the vertex nomination problem seeks to find the corresponding vertex of interest (if it exists) in a second network $G_2$.
A vertex nomination scheme produces a list of the vertices in $G_2$, ranked according to how likely they are judged to be the corresponding vertex of interest in $G_2$. 
The vertex nomination problem and related information retrieval tasks have attracted much attention in the machine learning literature, with numerous applications to social and biological networks. However, the current framework has often been confined to a comparatively small class of network models, and the concept of statistically consistent vertex nomination schemes has been only shallowly explored. In this paper, we extend the vertex nomination problem to a very general statistical model of graphs. Further, drawing inspiration from the long-established classification framework in the pattern recognition literature, we provide definitions for the key notions of Bayes optimality and consistency in our extended vertex nomination framework, including a derivation of the Bayes optimal vertex nomination scheme. In addition, we prove that no universally consistent vertex nomination schemes exist. Illustrative examples are provided throughout.
\end{abstract}


\section{Introduction}
\label{sec:intro}

%
%
%
%
Statistical inference on graphs is an important branch of modern statistics and machine learning.
In recent years, there have been numerous papers in the literature developing graph analogues of statistical inference tasks such as hypothesis testing \cite{asta2014geometric,tang2014nonparametric}, classification \cite{tang2013universally,chen2016robust}, and clustering \cite{von2007tutorial,rohe2011spectral,SusTanFisPri2012,NewCla2016}.
Moreover, growth in the size and complexity of network data sets have necessitated techniques for network-specific data mining tasks such as link prediction \cite{liben2007link,lu2011link}; entity resolution and network alignment \cite{ConteReview,lyzinski2016information};
and vertex nomination \cite{CopPri2012,Coppersmith2014,suwan2015bayesian,FisLyzPaoChePri2015,lyzinski2016consistency}.
Akin to the development of classical statistics, algorithmic advancement has, in many ways, outpaced theoretical developments in these emerging graph-driven domains.
This development has been necessitated by the dizzying pace of data generation,
but there is nevertheless the need for a firm theoretical context
in which to frame algorithmic progress.
Toward this end, in this paper, drawing inspiration from the long-established classification framework in the pattern recognition literature \cite{DGL}, we provide a rigorous theoretical framework for understanding statistical consistency in the vertex nomination  (VN) inference task.

\begin{figure}[t!]  
  \centering
  \includegraphics[width=0.50\columnwidth]{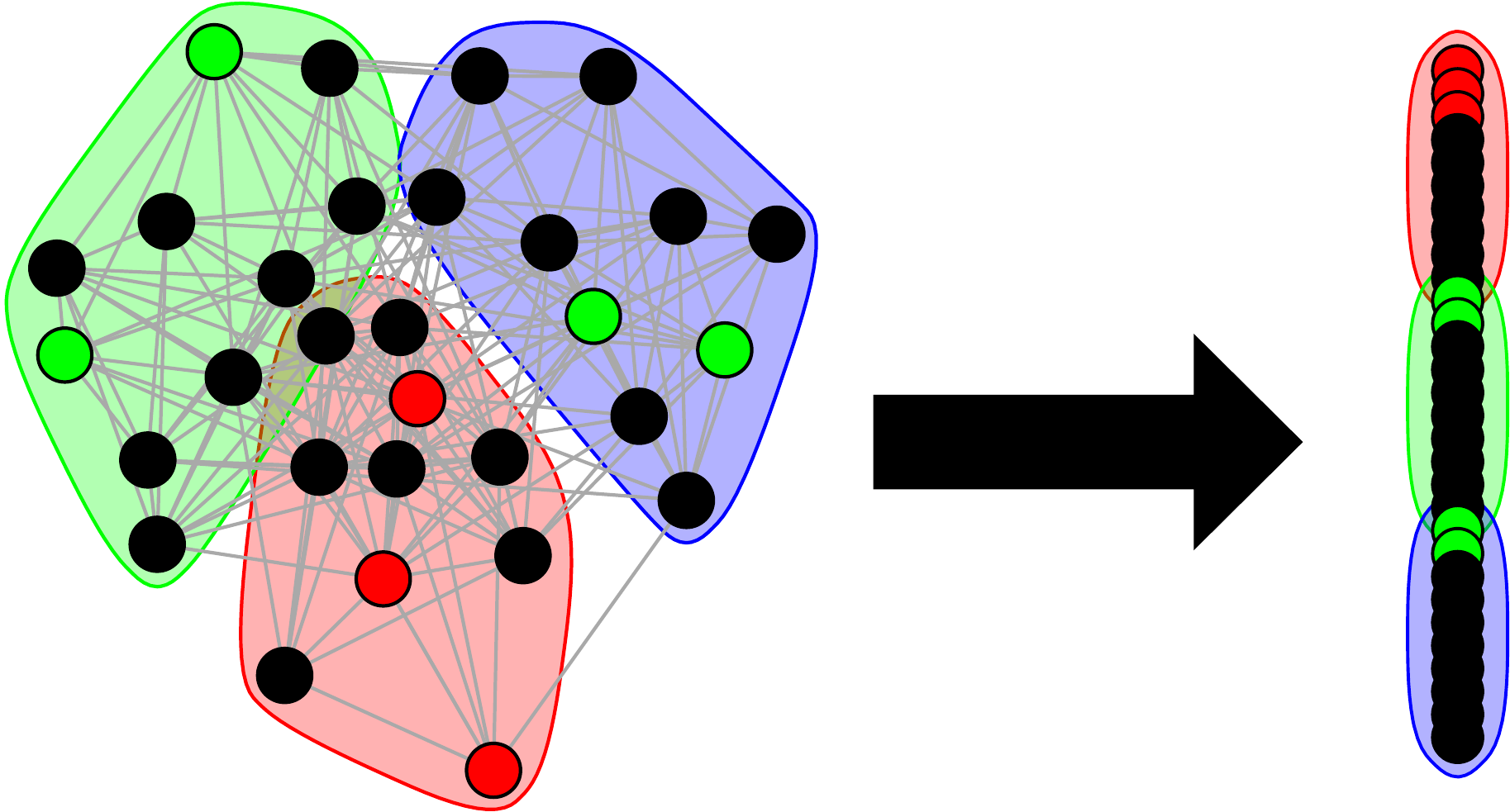}
  \caption{A visual representation of the classical Vertex Nomination framework:  Given a community of interest in a network (here the red community) and some examples of vertices that are/are not part of the community of interest (colored red and green, respectively), rank the remaining vertices in the network into a nomination list, with those vertices from the community of interest concentrating at the top of the nomination list.}
  \label{fig:VNv1}
\end{figure}

The inference task in vertex nomination, which can be viewed as the graph analogue of the more classical recommender system task \cite{ricci2011introduction}, has traditionally been stated as follows:
given a community of interest in a network and some examples of vertices that are
or are not part of a community of interest, vertex nomination seeks to rank the remaining vertices in the network into a nomination list, with those vertices from the community of interest (ideally) concentrating at the top of the nomination list.
See Figure~\ref{fig:VNv1} for a visual representation of this classical Vertex Nomination framework.
In limited-resource settings, vertex nomination tools have proven to be effective in efficiently searching and querying large networks, with
applications including
detecting fraudsters in the Enron email network \cite{CopPri2012,marchette2011vertex,suwan2015bayesian},
uncovering web advertisements that have association with human trafficking \cite{FisLyzPaoChePri2015}, and
identifying latent structure in connectome data \cite{FisLyzPaoChePri2015,yoder2018vertex}.

While related to the community detection problem \cite{Newman2006,von2007tutorial,BicChe2009,NewCla2016}, this traditional formulation of the VN problem is a semi-supervised inference task whose output is not an assignment of vertices to communities, but rather a ranked estimate of which vertices belong to a particular community of interest.
That is, in contrast to community detection, the VN problem does not aim to recover the community memberships of any vertices not in the community of interest.
Clearly, any method that can recover the community memberships of all vertices in a graph can recover the interesting community, and hence any community detection algorithm can be repurposed for the VN problem just described with minor adaptation (e.g., by ranking vertices according to their probability of membership in the community of interest); see, for example, the spectral vertex nomination scheme of \cite{FisLyzPaoChePri2015}.
The specific performance of such an adaptation is highly dependent on the fidelity of the base clustering procedure, and the performance is often below that of the semi-supervised VN specific analogues \cite{yoder2018vertex}.

The above formulation of the VN task assumes the presence of strong community structure among the vertices of interest in the graph.
In practice, this is often a reasonable assumption,
particularly if it is expected that interesting vertices will behave similarly to one another in the network.
However, the particular features that mark a vertex as interesting are entirely task-dependent.
To paraphrase the common proverb, interestingness is in the eye of the practitioner.
Interesting vertices may be, for example, those with large network centrality \cite{jeong2001lethality,newman2005measure}, those with a particular role in the network \cite{lusseau2004identifying}, or those corresponding to a given user across social networks \cite{patsolic2017vertex}.
In these applications, interesting vertices need not correspond precisely to the community structure captured by a generative network model, and hence such cases are ill-described by the community-based VN problem described above.
To accommodate this task-dependency and broader notion of interesting vertices,
we consider the following generalization and extension of the previously-presented VN problem:
Given a vertex of interest $v^*$ in a graph $G_1=(V_1,E_1)$,
find the corresponding vertex of interest $u^*$ (if it exists) in a second graph $G_2=(V_2,E_2)$ by ranking the vertices of $G_2$ according to our confidence that they correspond to $v^*$ in graph $G_1$; see Figure \ref{fig:VNv2} for a visual representation of this VN framework.
In this formulation, which is an (potentially) unsupervised inference task, what defines $v^*$ as interesting is entirely model-dependent, and different network models can highlight different characteristics of interest in the graph.
Potential application domains for this VN generalization abound, 
including identifying users of interest across social network platforms  (see, for example, \cite{patsolic2017vertex}), 
identifying structural signal across connectomes (see, for example, \cite{sussman2018matched}), 
and identifying topics of interest across graphical knowledge bases (see, for example, \cite{sun2013efficiency}).

\begin{figure}[t!]  
  \centering
  \includegraphics[width=0.75\columnwidth]{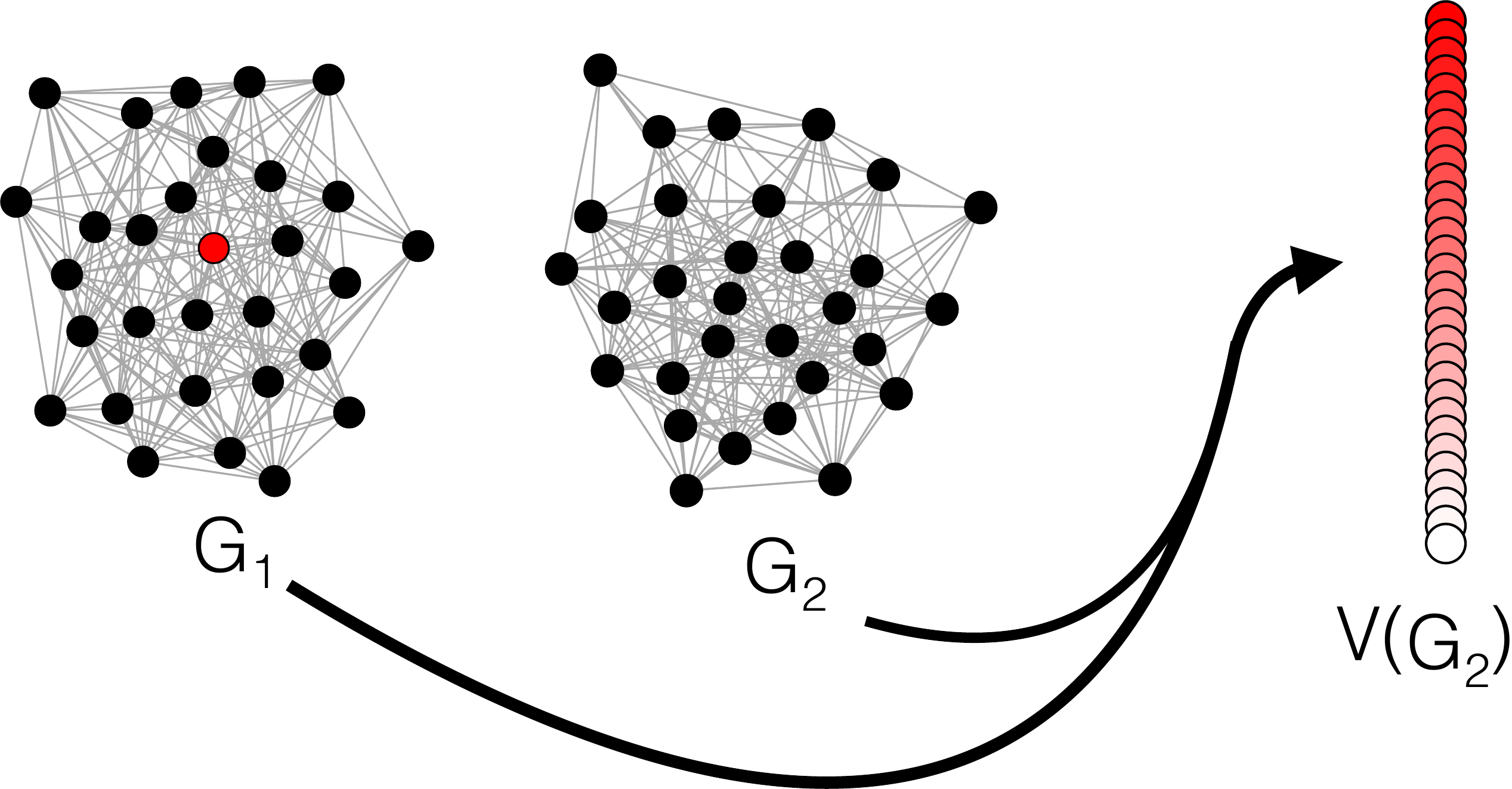}
  \caption{A visual representation of the generalized Vertex Nomination framework:  Given a vertex of interest $v^*$ (colored red) in a graph $G_1 = (V_1, E_1)$, find the corresponding vertex of interest $u^*$ (if it exists) in a second graph $G_2 = (V_2, E_2)$, ranking the vertices of $G_2$ into a nomination list so that $u^*$ ideally appears at the top of the nomination list.}
\label{fig:VNv2}
\end{figure}

In \cite{FisLyzPaoChePri2015} and \cite{lyzinski2016consistency}, the notion of a consistent vertex nomination scheme (i.e., an asymptotically optimal solution to the VN problem) was proposed for the
original formulation of the VN problem, in which community membership entirely determines whether or not a given vertex is interesting.
This definition of consistency was based on the \emph{mean average precision} (MAP) of a nomination scheme operating on a graph model with explicit community structure encoded by the the Stochastic Block Model (SBM) of \cite{sbm}.
Under this restricted notion of consistency, \cite{FisLyzPaoChePri2015} derived the analogue of universal Bayes optimality in the VN setting, namely a scheme that achieves the optimal mean average precision for all parameterizations of the underlying SBM.
While this derivation of the Bayes optimal scheme somewhat parallels the derivation of the Bayes optimal classifier in the classical pattern recognition literature, the SBM model assumption and MAP formulation greatly narrow the set of models and sets of interesting vertices we can consider.
In this paper, we revamp and generalize the concept of VN consistency---and of VN Bayes optimality---in the two-graph VN framework.
This framework is quite general, and further allows us to highlight the similarities and differences between our new VN consistency formulation and its analogue in the classification literature defined in, for example, \cite{DGL}.

The paper is laid out as follows.
In the remainder of this section, we provide brief overviews of information
retrieval as it relates to vertex nomination (Section~\ref{sec:IR}) and the
Bayes optimal classifier in the classical setting (Section~\ref{sec:BE}),
and conclude the introduction by establishing notation for the
remainder of the paper (Section~\ref{sec:notation}).
In Section \ref{sec:VN}, we define the VN problem framework that is the focus of this paper, and in Section \ref{sec:VNBEBO} we derive the VN analogue of a Bayes optimal scheme. In Section \ref{sec:VNCon}, we define a new notion of VN consistency, and we prove that no universally consistent VN scheme exists, providing an interesting contrast to the standard classification setting.
We conclude in Section \ref{sec:diss} with a short summary comparing and contrasting VN with classical classification and a discussion of implications and future directions.

\subsection{Connections to information retrieval}
\label{sec:IR}
The vertex nomination task is, in some ways, similar to the task faced by
recommender systems
\cite{resnick1997recommender,ricci2011introduction},
in which the aim is to retrieve objects (e.g., documents or images)
likely to be of interest to a user based on his or her previous behavior.
For example, the celebrated PageRank algorithm~\cite{BriPag1998}
recommends webpages based on random walks on the world wide web graph,
in which websites are nodes and (directed) edges reflect hyperlinks
between pages.
The information retrieval (IR) literature
includes many such graph-based approaches.
We refer the reader to \cite{ricci2011introduction} and \cite{MihRad2011}
for the state of the art circa 2010, and concentrate here on recapping
more recent graph-based information retrieval techniques.

Many graph-based IR techniques rely on the assumption that similar
objects (i.e., documents, webpages, etc.) will lie near one another
in a suitably-constructed graph,
an intuition underlying many graph-based approaches throughout machine
learning and related disciplines; see, for example,
\cite{BelNiy2003,ZhoBouLalWesSch2009}.
Techniques along these lines have been applied toward many tasks in
natural language processing, typically inspired by PageRank
\cite{RotSch2014}.
Along similar lines,
\cite{MaKinLyu2012} applies a diffusion-based method \cite{CoiLaf2006} to
the world wide web graph to yield an approach to ranking
for query completion and recommendation.
These information retrieval techniques can be naturally adapted to the vertex nomination problem by treating the vertex or vertices of interest as the object or objects to be retrieved.

The vertex nomination problem also bears similarities to the task of \emph{learning to rank} \cite{Duh2009,Liu2009,Li2011},
in which the goal is to learn an ordering on a set of objects
(i.e., documents, images, videos, etc.) according to (estimated) similarity
or relevance to a given query object.
In the learning to rank literature, graphs usually appear as training
instances, with nodes corresponding to objects
and edges encoding preferences or similarities among them elicited from users
(e.g., an undirected weighted edge may join two documents judged to be similar).
The work in \cite{AgaChaAgg2006} is among the earliest to consider the problem of ranking objects in a network.
The authors modified the PageRank algorithm to take
preference information into account, rather than working solely with
the hyperlink graph.
In \cite{Agarwal2010}, the authors use a \emph{data graph} encoding object similarities
to obtain a regularizer similar to \cite{BelNiySin2006}
on the empirical ranking error,
with the target ranking encoded in a \emph{preference graph}.
More recent efforts along these lines have focused on the problem of
incorporating network structure present between entities of different types,
for example, between users and events in a social network
\cite{LuoPanWanLin2014,PhaLiConZha2016}.
Here again, any learning to rank algorithm has a natural adaptation to the VN problem by using the  first graph, in which some vertices are labeled, as training data to learn a ranking on the vertices of the second graph.

\subsection{Bayes error in classical pattern recognition}
\label{sec:BE}
In this section, we review the concepts of consistency and Bayes error from the statistical classification literature. We do not aim to give an exhaustive overview of the subject, but only to provide a rough outline as to the structures that we would like to replicate in the context of vertex nomination.
For a more thorough treatment, we refer the interested reader to \cite{DGL}, whose presentation we follow below.

We begin by recalling the classical definition of Bayes error.
Note that we will restrict our attention to the two-class problem to maximally bring forth the similarities (and differences) between statistical classification and VN, as in VN vertices are either of interest or not.
\begin{definition}
\label{def:classifier}
Consider a set of potential observations $\mathcal{X}$ and a set of unknown class labels $\{0,1\}$ for objects in $\mathcal{X}$.
A classifier is a function $h:\mathcal{X}\rightarrow\{0,1\}$, which aims to predict the class label of a given observation in $\mathcal{X}$.
Given a distribution $F$ supported on $\mathcal{X}\times\{0,1\}$, the error for the classifier $h$ is given by $L(h)=\p( h(X) \neq Y )$ where $(X,Y)\sim F$.
\end{definition}
\noindent Any classifier that achieves the lowest possible error is said to be a {\em Bayes optimal} classifier.
We write $h^*$ for any such optimal classifier, which by definition satisfies
$ h^* \in \arg \min_{h : \calX \rightarrow \{0,1\}} L(h). $
It is easily seen in this two-class framework that the Bayes optimal classifier is given by
\begin{equation}
\label{eq:classbe}
h^*(x)=\begin{cases}
1&\text{ if }\,\mathbb{E}(Y|X=x)=\p(Y=1|X=x)>1/2;\\
0&\text{ else}.
\end{cases}
\end{equation}

Practically speaking, the Bayes optimal scheme chooses the label which maximizes the class-conditional probability of the observed data.
The corresponding error, $L^* = L(h^*)$, is called the \emph{Bayes error}.
Of course, $h^*$ depends on the distribution $F$ of $(X,Y)$, and, when appropriate, we will make this dependence explicit by writing $L^*_F$.

In practice, a classifier is often constructed based on training data
$(X_1,Y_1),(X_2,Y_2),$ $\dots,(X_n,Y_n)$, where the data
$(X_i,Y_i)$ are drawn i.i.d.~according to $F$.
This supervised classification framework is defined as follows.
\begin{definition}
\label{def:classifiern}
Consider a set of potential observations $\mathcal{X}$ and a set of unknown class labels $\{0,1\}$ for objects in $\mathcal{X}$.
A (supervised) classifier is a function
$$h_n:\mathcal{X}\times\{\mathcal{X}\times\{0,1\}\,\}^n\rightarrow\{0,1\},$$
which aims to predict the class label of a given observation in $\mathcal{X}$ based on $n$ training observations $(x_1,y_1),(x_2,y_2),\dots,(x_n,y_n)\in\mathcal{X}\times\{0,1\}$.
Given a distribution $F$ supported on $\mathcal{X}\times\{0,1\}$, the error for the classifier $h_n$ is given by
$$L_F(h_n)=\p\big[ h_n(X,(X_i,Y_i)_{i=1}^n) \neq Y \mid (X_i,Y_i)_{i=1}^n \big]$$
where $(X,Y),(X_1,Y_1),(X_2,Y_2),\dots,(X_n,Y_n)\stackrel{i.i.d.}{\sim} F.$
\end{definition}
\noindent Note that
$L_F(h_n)$ is a random variable in which $\{(X_i,Y_i)\}_{i=1}^n$ are drawn
i.i.d. from $F$, but then held fixed as we average over $(X,Y) \sim F$.

A sequence of classifiers $\mathbf{h}=(h_n)_{n=1}^\infty$ is called a {\em classification rule}.
Informally, a good classification rule is one for which the probability of error becomes arbitrarily close to Bayes optimal as $n\rightarrow\infty$.
The precise nature of what we mean by close is codified in the concept of statistical consistency.
\begin{definition}
\label{def:classifierconsis}
A classification rule $\mathbf{h}=(h_n)_{n=1}^\infty$ is {\em consistent} with respect to $F$ if
$$\mathbb{E}_F(L(h_n))\rightarrow L^*_F.$$
The rule $\bf h$ is {\em strongly consistent} if
$$L_F(h_n)\stackrel{a.s.}{\rightarrow} L^*_F.$$
A rule that is (strongly) consistent for all distributions $F$ on $\calX \times \{0,1\}$ is called \emph{(strongly) universally consistent}.
\end{definition}
\noindent
Perhaps surprisingly, given that $F$ can have arbitrary structure on $\mathcal{X}\times\{0,1\}$, universally consistent classification rules exist; see
\cite{Stone1977} for the first proof of this phenomenon.

In \cite{FisLyzPaoChePri2015}, a notion of consistency for vertex nomination was presented, roughly analogous to Definition~\ref{def:classifierconsis}.
In contrast to the classification task presented above, vertex nomination requires a ranking of the vertices, rather than merely the classification of a single vertex.
As such, a vertex nomination scheme is evaluated
in \cite{FisLyzPaoChePri2015} based on
\emph{average precision} \cite{ManRagSch2008},
rather than simply a fraction of correctly-classified vertices.
In \cite{FisLyzPaoChePri2015}, VN consistency is defined in the context of stochastic block model (SBM) random graphs with respect to a
provably optimal \emph{canonical} nomination scheme.
This canonical scheme plays an analogous role of Bayes optimal classifiers
in this restricted model framework (see Section~\ref{sec:VNBEBO} below).
The goal of this paper is to explore and further develop
a broader notion of VN consistency that encompasses a more expressive class of models than the SBM.

\subsection{Notation and background} \label{sec:notation}
We conclude this section by establishing notation and reviewing a few of the more popular statistical network models that we will make use of as examples in the sequel.

\subsubsection{Notation}

For a set $S$, we let $|S|$ denote its cardinality and
$\binom{S}{2}$ denote the set of all unordered pairs of distinct elements from $S$.
Throughout, we will denote graphs via the ordered pair $G=(V,E)$,
with vertices $V$ and edges $E \subseteq \binom{V}{2}$.
All graphs considered herein will be labeled, hollow (i.e., containing
no self-edges), and undirected.
We let $\calG_n$ denote the set of all labeled, hollow, undirected graphs on $n$ vertices.
Given a graph $G$, we will let $V(G)$ denote the vertices of $G$ and $E(G)$ denote its edges.
We note that when $G$ is random, this latter set is a random subset of $\binom{V}{2}$.
For a set of vertices $S \subseteq V(G)$, we let $G[S]$ denote the subgraph of $G$ induced by $S$,
i.e., the graph $G' = (S,E)$ with $\{u,v\} \in E$ if and only if $\{u,v\} \in E(G)$.
In a few places, we will require the notion of an {\em asymmetric} graph.
A graph $G \in \calG_n$ is {\em asymmetric} if it has no nontrivial automorphisms \cite{ErdRen1963}.
For a positive integer $n\in\mathbb{Z}$, we will define $[n] = \{1,2,\dots,n\}$, and $\gn$ to be the be the set of labeled graphs on $n$ vertices.
Throughout this paper, we will often, in order to simplify notation, suppress dependence of parameters on $n$.
Throughout, the reader should assume that, unless specified otherwise, all parameters depend on the number of vertices $n$.

\subsubsection{Stochastic block models}
The stochastic block model (SBM) is a widely studied model for edge-independent random graphs with latent community structure \cite{sbm,Hoff2002,karrer11:_stoch}.
\begin{definition}
\label{def:sbm}
We say that a random graph $G = (V,E) \in \calGn$ is an instantiation of a stochastic block model with parameters $(K,B,b)$, written $G\sim \SBM(K,B,b)$, if
\begin{itemize}
\item[i.] $V$ is partitioned into $K$ classes (called communities or blocks), $V = V_1 \cup V_2 \cup \dots \cup V_K$.
\item[ii.] The block membership vector $b \in [K]^{|V|}$ is such that for all $k \in [K]$, $b_v = k$ if and only if $v \in V_k$.
\item[iii.] The symmetric matrix $B \in [0,1]^{K \times K}$ denotes the edge probabilities between and within blocks, with
$$\mathbbm{1}_{\{\,\{u,v\} \in E(G)\}}\stackrel{ind.}{\sim}\text{Bernoulli}(B_{b_u,b_v}).$$
\end{itemize}
\end{definition}
\noindent We note that when $K=1$, we recover the \ErdosRenyi random graph \cite{Erdos}, in which the edges of $G$ are present or absent independently with probability $p$.
In this special case, we write $G \sim \ER(n,p)$.
By a slight abuse of notation, for a symmetric matrix $P \in [0,1]^{n \times n}$, we will write $G \sim \ER(P)$ if, identifying the vertices of $G$ with $[n]$, we have $\{i,j\} \in E(G)$ with probability $P_{i,j}$ independent of the other edges.
With no restrictions on $P$, $\ER(P)$
random graphs can be viewed as $n$-block SBMs and are the most general edge-independent random graph model.

The latent community structure inherent to SBMs makes them a natural model for use in the traditional vertex nomination framework. Recall the traditional VN task:  given a community of interest in a network and some examples of vertices that are or are not part of the community of interest, vertex nomination seeks to rank the remaining vertices in the network into a nomination list, with those vertices from the community of interest (ideally) concentrating at the top of the nomination list.
As a result, previous work on VN consistency \cite{FisLyzPaoChePri2015} has been posed within the SBM framework, with the optimal scheme only obtaining its optimality for SBMs.
We note that we consider herein the SBM setting where communities are disjoint and each vertex can only belong to a single community.
However, the results contained herein translate immediately to the mixed membership SBM setting \cite{Airoldi2008}; details are omitted for brevity.

\subsubsection{Random dot product graphs}
\label{sec:rdpg}

In stochastic block models, the block assignment vector can be viewed as a latent feature vector for the vertices in the network, with these features (i.e., block memberships) defining the connectivity structure in the network.
The random dot product graph (RDPG) model \cite{young2007random} allows for more nuanced vertex features to be incorporated into the model and has been used as the setting for a VN formulation similar to the one proposed here; see \cite{patsolic2017vertex} for details.
\begin{definition}
\label{def:rdpg}
We say that a random graph $G = (V,E) \in \calGn$ is an instantiation of a $d$-dimensional random dot product graph with parameters $X$, written $G\sim \RDPG(X)$, if
\begin{itemize}
\item[i.] The matrix $X \in \R^{n \times d}$ is such that $0 \le (X X^T)_{i,j} \le 1$ for all $i,j \in [n]$.
The rows of $X$ provide the latent features for the vertices in $V$.
\item[ii.] The edges of $G$ are present or absent independently, with $\{i,j\} \in E(G)$ with probability $(X X^T)_{i,j}$.
Written succinctly, $G\sim\ER(XX^T)$.
\end{itemize}
\end{definition}
\noindent We can view the RDPG model as a example of the more general \emph{latent position} random graph model \cite{Hoff2002}, in which edge probabilities are determined by hidden vertex-level geometry.

Estimating the latent position structure in RDPGs is particularly amenable to spectral methods, and this model has played a prominent role in recent theoretical developments of spectral graph methods; see, for example, \cite{rohe2011spectral,SusTanFisPri2012,MT2}.
Note that the RDPG can be extended to a broader class of models, in which edge probabilities are given by evaluating a positive definite link function at vertices' latent positions as in, for example, \cite{tang2013universally}.
While incorporating this more general family of latent position graphs into the present VN framework would be straightforward, we restrict our focus to the RDPG model of Definition \ref{def:rdpg} for ease of exposition.

\subsubsection{Correlation across networks}
\label{sec:corr}

The vertex nomination problem we consider in this paper presupposes the existence of a vertex of interest in a network $G_1$ and, ideally, a corresponding vertex of interest in a second network $G_2$.
Often, such correspondences across networks are encoded into random graph models via edge-wise graph correlation; see, for example, \cite{ModFAQ}.
Arguably the simplest such structure is seen in the $\rho$-correlated \ErdosRenyi model of \cite{rel}.
\begin{definition}
We say that bivariate random graphs $(G_1,G_2) \in\calGn\times\calGn$ are an instantiation of a $\rho$-correlated $\ER(P)$ model, written $(G_1,G_2) \sim \rho\text{-}\ER(P)$, if
\begin{itemize}
\item[i.] Marginally, $G_1\sim \ER(P)$ and $G_2\sim \ER(P)$.
\item[ii.] Edges are independent across $G_1$ and $G_2$ except that the indicators of the events $\{u,v\} \in E(G_1)$ and $\{u,v\} \in E(G_2)$ are jointly distributed as a pair of Bernoulli random variables with success probability $P_{u,v}$ and correlation $\rho$.
If the correlation is allowed to vary across edges, so that
these two events have correlation $\rho_{u,v}$, then collecting these correlations in a symmetric matrix
$R = [\rho_{i,j}]_{i,j=1}^n$, we write $(G_1,G_2) \sim \RER(P)$; see \cite{unmatchable}.
\end{itemize}
\end{definition}
\noindent Ranging the values in $R$ from $0$ to $1$ allows for the consideration of graphs that range from independent ($R=0$) to isomorphic ($R=1$).
Intermediate values of $R$ allow for the encoding of a correspondence across networks between these two extremes.
We will also consider $R<0$, in which case edges across networks are anti-correlated.
This is particularly useful for modeling situations in which corresponding vertices stochastically behave differently across networks.


\section{Vertex Nomination}
\label{sec:VN}
Loosely stated, the vertex nomination problem we consider in this paper can be summarized as follows:  Given a vertex of interest $v^*$ in a graph $G_1=(V_1,E_1)$,
find the corresponding vertex of interest $u^*$ (if it exists) in a second graph $G_2=(V_2,E_2)$
by ranking the vertices of $G_2$ according to our confidence that they correspond to $v^*$ in graph $G_1$.
To formally define this version of vertex nomination, we will need to consider distributions on graphs with partially-overlapping node sets that have a built-in notion of vertex correspondence across graphs.
To this end, we will consider distributions on $\gn\times\gm$, where $\gn$ is the set of labeled graphs on $n$ vertices, with vertex labels given by $\{v_1,v_2,\ldots,v_n\}$, and $\gm$ is the set of labeled graphs on $m$ vertices, with vertex labels given by $\{u_1,u_2,\ldots,u_m\}$.
Note that for $i\in[n]\cap[m]$, $v_i$ and $u_i$ are merely vertex labels, and it is not necessarily the case that $v_i=u_i$.
We follow this labeling convention in order to emphasize the reality that the vertex sets of $G_1$ and $G_2$ may only partially overlap.

\begin{definition}[Nominatable Distributions]
\label{def:nd}
We define the family of \emph{Nominatable Distributions}, which we denote $\mathcal{N}$, to be the family of distributions
$$\mathcal{N}=\left\{F_{c,n,m,\theta}\text{ s.t. }n,m\in\mathbb{Z}^+,\,\theta\in\Theta,\right\},$$ where
$F_{c,n,m,\theta}$ is a distribution on $\calGn\times\gm$ parameterized by $\theta \in \Theta$ and satisfying:
\begin{itemize}
\item[i.] The vertex sets $V_1=\{v_1,v_2,\ldots,v_n\}$ and $V_2=\{u_1,u_2,\ldots,u_m\}$ satisfy $v_i=u_i$ for $0<i\leq c$.
We refer to $C=\{v_1,v_2,\ldots,v_c\}=\{u_1,u_2,\ldots,u_c\}$ as the {\em core vertices}. These are the vertices that are shared across the two graphs and imbue the model with a natural notion of corresponding vertices.
\item[ii.] Vertices in $J_1=V_1\setminus C$ and $J_2=V_2\setminus C$, satisfy $J_1 \cap J_2=\emptyset$.
We refer to $J_1$ and $J_2$ as \emph{junk} vertices. These are the vertices in each graph that have no corresponding vertex in the other graph.
\item[iv.] The induced subgraphs $G_1[J_1]$ and $G_2[J_2]$ are conditionally independent given $\theta$.
\end{itemize}
\end{definition}

A few examples will serve to illustrate this definition. We will return to the three example settings below several times throughout the rest of the paper in order to highlight and illustrate phenomena of interest.

\begin{example}[$\RER(P)$] \label{ex:RERP} \emph{
Let $(G_1,G_2)\sim \RER(P)$ with $P,R\in\R^{n\times n}$ and $R > 0$ entrywise,
so that $G_1$ and $G_2$ have correlated edges as described in Section \ref{sec:corr}.
In this example, the model parameter is $\theta=(P,R)$,
and the vertex sets of the two graphs can be thought of as fully overlapping, i.e., $V_1=V_2=C=[n]$ and $J_1=J_2=\emptyset$,
since the correlation structure conveyed in the entries of $R$ encodes an explicit correspondence
between the edges of $G_1$ and the edges of $G_2$ (and hence also a correspondence between
$V_1$ and $V_2$).
Note that if we consider $C=[k]$ with $k<n$, then we would require (after suitably ordering the vertices) $R_{u,v}=0$ for $u,v>k$.
This highlights the way in which $\theta$ (and hence the distribution $F_{c,n,m,\theta}$)
can vary with $c$, and vice-versa.
}\end{example}

\begin{example}[RDPG] \label{ex:RDPG} \emph{
Let $m > n$ and suppose that $Y \in \R^{m \times d}$ has distinct rows
and satisfies $(Y Y^T)_{i,j} \in [0,1]$ for all $i,j \in [m]$.
Let $X \in \R^{n \times d}$ be a submatrix of $Y$,
and consider $G_1 \sim \RDPG(X)$ and $G_2\sim \RDPG(Y)$,
where $G_1$ and $G_2$ are conditionally independent given $Y$.
In this example, we can consider $\theta=Y$, $V_1=[n]=C$, $J_1=\emptyset$, $V_2=[n]\cup J_2$, and $J_2=\{u_{n+1},u_{n+2},\ldots,u_m\}$.
Note that as $G_1$ and $G_2$ are conditionally independent given $\theta$, we could also consider $0<c<n$ here as well. This illustrates that $\theta$ need not necessarily vary with $c$, and hence $F_{c,n,m,\theta}$ need not vary with $c$, either.
}\end{example}

\begin{example}[Independent \ErdosRenyi graphs]\label{ex:indERgraphs}\emph{
Let $(G_1,G_2)$ be independent $\ER(n,p)$ random graphs.
In this example, we can consider any $c\in [n]$.
Note that if $c=0$ here, then there is no corresponding vertex of interest in $G_2$, and this example serves as a natural boundary case between models in which nomination is possible and those
in which it is not.
As we will see below in Theorem~\ref{thm:kconsis},
$c>0$ may still yield chance performance for any nomination scheme,
and the existence of a vertex correspondence does not necessarily imply any performance guarantees.
}\end{example}

\begin{remark}
\emph{
In addition to the edge-independent and conditionally edge-independent network models considered above, the class of nominatable distributions
contains a host of other popular random graph models, including the Exponential Random Graph Model \cite{ERGM,ERGM2,ERGM3}, the preferential attachment model \cite{PA}, and the Watts-Strogatz small world model \cite{SW_WS}, among others.
Indeed, if we consider the case where $c=n=m$, then any parametric distribution on $\mathcal{G}_n\times\mathcal{G}_n$ is a nominatable distribution.
}
\end{remark}

\begin{remark}
\emph{
The core vertices $C$ in a nominatable distribution correspond to the vertices that can be sensibly identified across graphs. Note that this set does not require any further structure, aside from the conditional independence of $G_1[J_1]$ and $G_2[J_2]$ given the parameter $\theta$. Thus, we are largely free to specify any notion of correspondence we please. Depending on the application, this correspondence may be that of vertices playing similar structural roles, belonging to the same community, or some more complicated application-specific notion of correspondence. That is, the notion of cross-graph correspondence, and hence the notion of vertex similarity, is largely left to the practitioner to specify when she or he specifies an appropriate random graph model.
}
\end{remark}

Given a pair of graphs $(G_1,G_2) \sim F_{c,n,m,\theta}\in\mathcal{N}$, if the vertices in $C$ are known across graphs then identifying the corresponding vertex to $v^* \in C$ is immediate from the vertex labels.
In practice, this information is unknown and the correspondences across graphs are only partially observed or even unobserved entirely.
To model this added uncertainty, we consider passing the vertex labels of $G_2$ through an obfuscating function.
\begin{definition}
Let $(G_1,G_2) \sim F_{c,n,m,\theta}\in\mathcal{N}$.
An \emph{obfuscating function} $\mo:V_2\rightarrow W$ is a bijection from $V_2$ to $W$ with $W\cap V_i=\emptyset$ for $i=1,2$.
We call a set $W$ satisfying $W\cap V_i=\emptyset$ for $i=1,2$ an obfuscating set, and for a given obfuscating set $W$, we let $\mathfrak{O}_W$ be the set of all obfuscating functions $\mo:V_2\rightarrow W$.
\end{definition}
\noindent Here, $\mo$ models the practical reality that the correspondence of labels across graph is unknown a priori.
Note that to ease notation, we shall write $\mo(G_2)$ (resp., $\mo(g_2)$ and $\mo(\gm)$) to denote the graph $G_2$ (respectively, $g_2$ and $\gm$) whose labels have been obfuscated via $\mo$.

\begin{figure}[t!]  
  \centering
  \includegraphics[width=0.6\columnwidth]{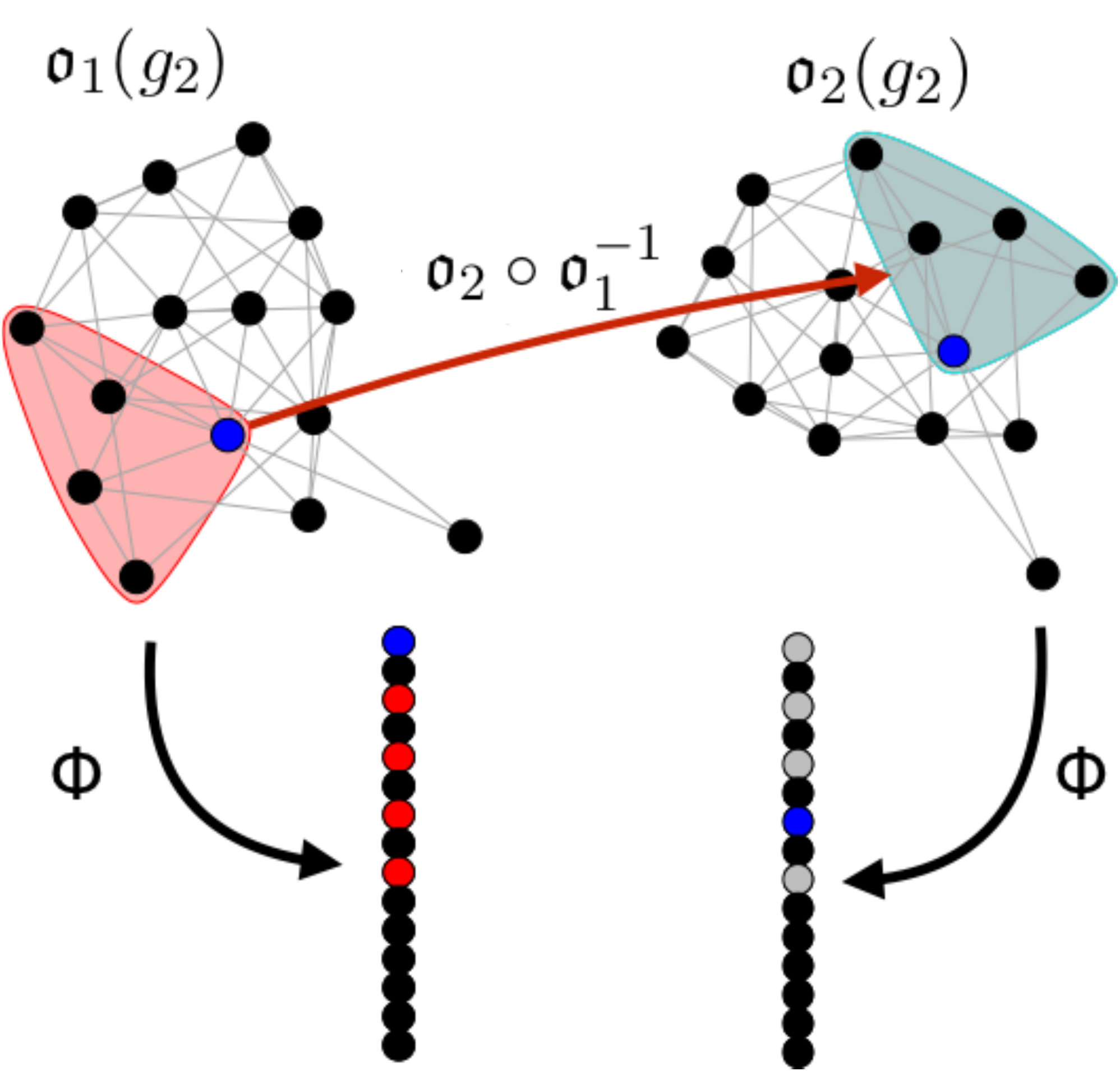}
  \caption{An illustration of the ``label-independence'' property of VN schemes.
  If the blue vertex in $\mo_1(g_2)$ (resp., $\mo_2(g_2)$) is $\mo_1(u)$ (resp., $\mo_2(u)$) for $u\in V_2$, then we require the ranks of $\mathcal{I}(\mo_1(u);\mo_1(g_2))$ (outlined in red in the network $\mo_1(g_2)$ and colored red/blue in the ordering provided by $\Phi$) to be equal to the ranks of $\mathcal{I}(\mo_2(u);\mo_2(g_2))$ (outlined in grey in the network $\mo_2(g_2)$ and colored grey/blue in the ordering provided by $\Phi$).  
Indeed, the set of ranks of $\mo(\mathcal{I}(u;g_2))$ via $\Phi$ is independent of the choice of obfuscation function $\mo$.
  }
\label{fig:consis2}
\end{figure}

Before defining a VN scheme, we must make one additional definition: for a graph $g\in\gm$ and $u\in V(g)$, define
$$
\mathcal{I}(u;g)=\{w\in V(g)\text{ s.t. } 
	\exists\text{ an automorphism }\sigma\text{ of }g,\text{ s.t. }\,\sigma(u)=w\}.
$$
Note that by taking $\sigma$ to be the identity, we have $u\in \mathcal{I}(u;g)$.
The vertices in $\mathcal{I}(u;g)$ are those that are, in a sense, topologically equivalent to the vertex $u$ in $g$, and hence, in the absence of labels, indistinguishable from one another.
As such,
any sensibly-defined vertex nomination scheme should view all vertices in $\mathcal{I}(u;g)$ as being equally good matches to a vertex of interest $v^*$.
Thus, a well-defined VN scheme should be ``label-independent'' in the following sense:
The set of ranks of each set of equivalent vertices (i.e., each $\mathcal{I}(u;g_2)$) needs to be invariant to the particular choice of obfuscating function; see Figure \ref{fig:consis2} for an illustration of this consistency criterion.
Formally, we have the following.

\begin{definition}[Vertex Nomination (VN) Scheme]\label{def:VN}
For a set $A$, let $\calT_A$ denote the set of all total orderings of the elements of $A$.
For $n,m>0$ fixed and obfuscating set $W$, a \emph{vertex nomination scheme} is a function $\Phi: \gn \times \gm \times \mathfrak{O}_W \times V_1 \rightarrow \calT_{W}$ satisfying the following consistency property:
If for each $u\in V_2$, we define $\text{rank}_{\Phi(g_1,\mo(g_2),v^*)}\big(\mo(u)\big)$ to be the position of $\mo(u)$ in the total ordering provided by $\Phi(g_1,\mo(g_2),v^*)$, and we define
$\mathfrak{r}_{\Phi}:\gn\times\gm\times\mathfrak{O}_W\times V_1\times2^{V_2}\mapsto 2^{[m]}$ via
$$\mathfrak{r}_{\Phi}(g_1,g_2,\mo,v^*,S)=\{\text{rank}_{\Phi(g_1,\mo(g_2),v^*)}\big(\mo(u)\big)\text{ s.t. }u\in S \},$$
then we require that for any $g_1\in\gn,$ $g_2\in\gm$, $v^*\in V_1$, obfuscating functions $\mo_1,\mo_2\in\mathfrak{O}_W$ and any $u\in V(g_2)$, 
\begin{align}
\label{eq:consis}
&\mathfrak{r}_{\Phi}\big(g_1,g_2,\mo_1,v^*,\mathcal{I}(u;g_2)\big)
=\mathfrak{r}_{\Phi}\big(g_1,g_2,\mo_2,v^*,\mathcal{I}(u;g_2)\big)\\
\notag &\Leftrightarrow \mo_2\circ\mo_1^{-1}\Big[ \mathcal{I}\Big(\Phi(g_1,\mo_1(g_2),v^*)[k]);\mo_1(g_2)\Big)\Big]=
\mathcal{I}\Big( \Phi(g_1,\mo_2(g_2),v^*)[k];\mo_2(g_2) \Big),\text{ for all }k\in[m],
\end{align}
where $\Phi(g_1,\mo(g_2),v^*)[k]$ denotes the $k$-th element (i.e., the rank-$k$ vertex) in the ordering $\Phi(g_1,\mo(g_2),v^*)$.
We let $\calVnm$ denote the set of all such VN schemes.
\end{definition}

\noindent Given $(G_1,G_2)\sim F_{c,n,m,\theta}\in\mathcal{N}$ realized as $G_1=g_1$ and $G_2=g_2$ with $v^*\in V_1$ the vertex of interest in $G_1$, a VN scheme $\Phi(\cdot,\cdot,\cdot)$ produces a ranked list $\Phi(g_1,\mathfrak{o}(g_2),v^*)$ of the vertices of $\mathfrak{o}(g_2)$ (i.e., the set $W$), ordered according to how likely each vertex in $V(\mo(g_2))$ is judged to correspond to $v^*$,
with optimal performance corresponding to
$$\Phi(g_1,\mo(g_2),v^*)[1]=\begin{cases}\mo(v^*)&\text{ if }v^*\in C\\
\text{arbitrary }v\in W&\text{ if }v^*\notin C.
\end{cases}$$
Less formally, one can think of a VN scheme as ranking the vertices of $G_2$ according to
how well they resemble the vertex of interest $v^*$
under some task-dependent measure.

\begin{remark}
\emph{
Note that if $u\in V_2$ is such that
$\mathcal{I}(u;g_2)= \{u\}$ (i.e., $u$ is topologically distinct within $g_2$), then Equation \eqref{eq:consis} implies that
$$\text{rank}_{\Phi(g_1,\mo_1(g_2),v^*)}\big(\mo_1(u)\big)=\text{rank}_{\Phi(g_1,\mo_2(g_2),v^*)}\big(\mo_2(u)\big)$$
for any $\mo_1$, $\mo_2$ in $\mathfrak{O}_W$.
If $\mathcal{I}(u;g_2)$ contains vertices in addition to $u$, then Equation \eqref{eq:consis} implies that the set of vertices topologically equivalent to
$u$ (namely, those in $\mathcal{I}(u;g_2)$) must achieve the same ranks via $\Phi$ under any two obfuscating functions; see Figure \ref{fig:consis} for a simple example of this consistency criterion in action.
}
\begin{figure}[t!]  
  \centering
  \includegraphics[width=\columnwidth]{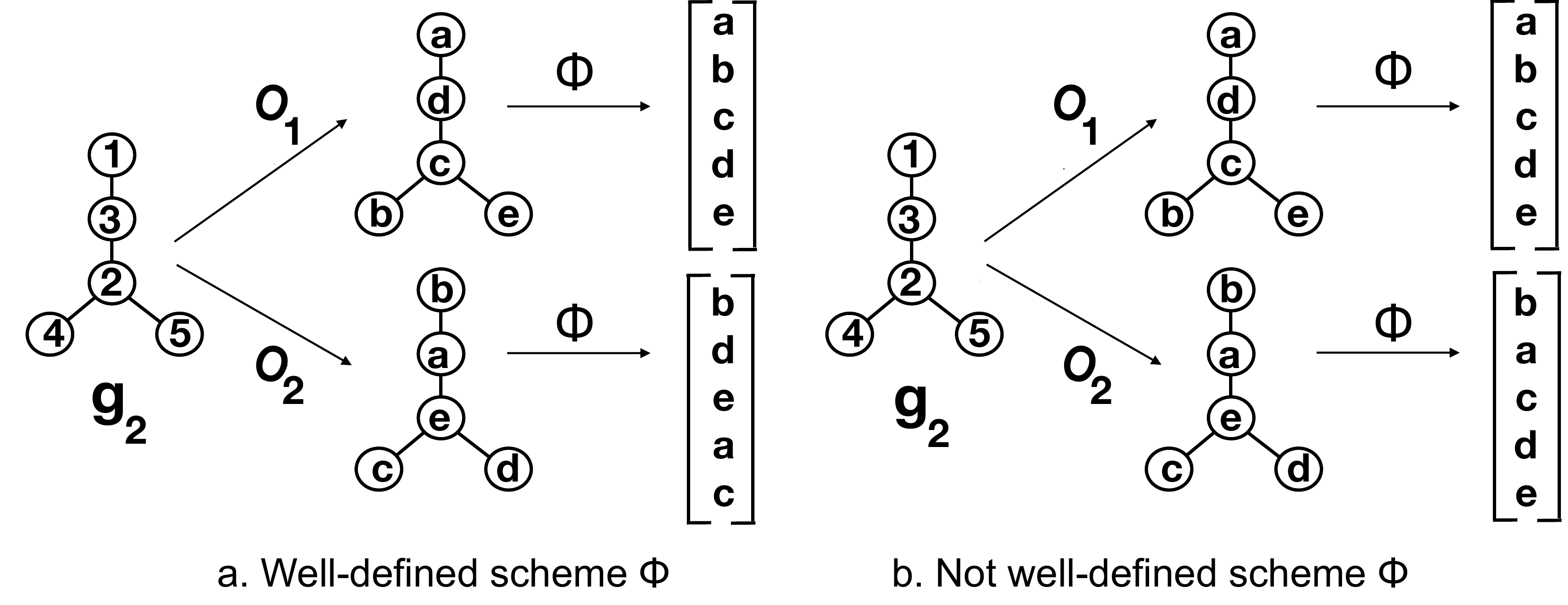}
  \caption{An example of the consistency criterion, Equation \eqref{eq:consis}, in action.  
  The left panel (a) shows a well-defined nomination scheme while the right panel (b) shows an ill-defined scheme.  
  The key in this example is that any scheme satisfying Equation \eqref{eq:consis} must have 
  $\mathfrak{r}_{\Phi}(g_1,g_2,\mo_1,v^*,\{1\})=\mathfrak{r}_{\Phi}(g_1,g_2,\mo_2,v^*,\{1\})$; $\mathfrak{r}_{\Phi}(g_1,g_2,\mo_1,v^*,\{3\})=\mathfrak{r}_{\Phi}(g_1,g_2,\mo_2,v^*,\{3\})$; $\mathfrak{r}_{\Phi}(g_1,g_2,\mo_1,v^*,\{2\})=\mathfrak{r}_{\Phi}(g_1,g_2,\mo_2,v^*,\{2\})$; and 
  $\mathfrak{r}_{\Phi}(g_1,g_2,\mo_1,v^*,\{4,5\})=\mathfrak{r}_{\Phi}(g_1,g_2,\mo_2,v^*,\{4,5\})$.
  }
\label{fig:consis}
\end{figure}
\end{remark}

\begin{remark}[Relation to \cite{FisLyzPaoChePri2015,lyzinski2016consistency}]
\label{rem:one-graph}
\emph{Recall \\the one-graph vertex nomination task considered in earlier
works \cite{Coppersmith2014,FisLyzPaoChePri2015,lyzinski2016consistency}
and described in Section~\ref{sec:intro},
in which vertices are considered interesting precisely when they belong to one of $K$ communities in a stochastic block model.
While the two-graph VN formulation we consider in the present work (modulo symmetries) involves a single vertex of interest across graphs, the framework is easily extended to the setting where one may have multiple vertices of interest (and not of interest), and in particular it can encode instances of the one-graph version VN problem.
To see this, consider an instance of the single-graph VN problem on graph
$G = (V, E)$ where $V$ is partitioned into $K$ communities as
$V = V_1 \cup V_2 \cup \dots \cup V_K$
and each of the communities is comprised of labeled (i.e., seed vertices, whose community memberships are observed)
and unlabeled (i.e., nonseed, whose community memberships are unobserved) vertices,
$V_k = S_k \cup U_k$, where $S_k \subseteq V_k$ is the set of seeds from
the $k$-th block and $U_k \subseteq V_k$ is the set of nonseed vertices.
We can encode this one-graph VN instance as an instance of the
two-graph problem by encoding additional information in the graph $G_1$.
Construct a vertex set  $V' = V \cup \{\ell_1,\ell_2,\dots,\ell_K\}$.
The $K$ new vertices $\{ \ell_k \}_{k=1}^K$ will encode the label
information present in the graph $G$.
Let $E' = E \cup L$, where $L = \{ \{\ell_k, s\} : s \in S_k, k \in [K] \}$,
so that edges connect from seed vertices in $S=S_1\cup S_2 \cup \dots \cup S_K$
to their corresponding label vertices.
Take $G_1 = (V',E')$, and let the interesting vertices
(and possible uninteresting vertices) be given by the elements of
$S \subseteq V'$.
The second graph $G_2$ is then
the subgraph of $G$ induced by the unlabeled vertices $U \subseteq V$ passed through an appropriate obfuscating function.
This pair $(G_1,G_2)$, with any $s \in S_1$ chosen to be the interesting
vertex, encodes the label information present in the one-graph VN problem
as well as the graph structure of $G$, as required.}
\end{remark}

\section{Bayes error and Bayes optimality in Vertex Nomination}
\label{sec:VNBEBO}

Viewing a VN scheme as an information retrieval system suggests
that a scheme that puts $\mo(v^*)$ close to the top of the nomination list is potentially of great practical value, even if it fails to obtain perfect performance.
Motivated by this, we adapt the recall-at-$k$ metric from classical information retrieval as a measure of performance.
To wit, we define the level-$k$ loss function and error for VN as follows.
\begin{definition}[VN loss function, level-$k$ error]
\label{def:lossfcn}
Let $\Phi\in\calVnm$ be a vertex nomination scheme and $\mo$ an obfuscating function.  
For $(g_1,g_2)$ realized from $(G_1,G_2)\sim F_{c,n,m,\theta}$ with vertex of interest $v^*\in C$, and for $k\in[m-1]$, we define the \emph{level-$k$ nomination loss} via
\begin{equation} \label{eq:lossfcn}
\begin{aligned}
\ell_k(\Phi,g_1,g_2,v^*)&=\mathds{1}\{ \text{rank}_{\Phi(g_1,\mo(g_2),v^*)}(\mo(v^*)) \ge k+1\} \\
&=1-\mathds{1}\{ \text{rank}_{\Phi(g_1,\mo(g_2),v^*)}(\mo(v^*)) \leq k\}.
\end{aligned} \end{equation}
The \emph{level-$k$ error} of $\Phi$ at $v^*$ is then defined to be 
\begin{equation} \label{eq:vnloss}
\begin{aligned}
L_k( \Phi, v^* ) &=\mathbb{E}_{(G_1,G_2)\sim F_{c,n,m,\theta}}\left[ \ell_k(\Phi,G_1,G_2,v^*) \right] \\
&=\mathbb{P}_{(G_1,G_2)\sim F_{c,n,m,\theta}}\left[ \text{rank}_{\Phi(G_1,\mo(G_2),v^*)}(\mo(v^*)) \ge k+1 \right].
\end{aligned} \end{equation}
\end{definition}
\noindent From the definition of the level-$k$ error in Eq. (\ref{eq:vnloss}), it is immediate that
\begin{equation} \label{eq:Lchain}
\begin{aligned}
L_1( \Phi, v^* )&=1-\mathbb{P}_{(G_1,G_2)\sim F_{c,n,m,\theta}}\left[ \text{rank}_{\Phi(G_1,\mo(G_2),v^*)}(\mo(v^*))=1\right]\\
 &\geq L_2( \Phi, v^* )=1-\mathbb{P}_{(G_1,G_2)\sim F_{c,n,m,\theta}}\left[ \text{rank}_{\Phi(G_1,\mo(G_2),v^*)}(\mo(v^*))\in\{1,2\}\right] \\
 &\geq L_3( \Phi, v^* )=1-\mathbb{P}_{(G_1,G_2)\sim F_{c,n,m,\theta}}\left[ \text{rank}_{\Phi(G_1,\mo(G_2),v^*)}(\mo(v^*))\in\{1,2,3\}\right]\\
 &\hspace{10mm}\vdots \\
 &\geq L_{m-1}( \Phi, v^* ) =1-\mathbb{P}_{(G_1,G_2)\sim F_{c,n,m,\theta}}\left[ \text{rank}_{\Phi(G_1,\mo(G_2),v^*)}(\mo(v^*))\in[m-1]\right],
\end{aligned} \end{equation}
The level-1 loss function is analogous to the classical 0/1 loss function in classification, as $L_1( \Phi, v^* )$ is simply the probability that $\Phi$ fails to ``classify'' $\mo(v^*)$ as the vertex corresponding to $v^*$ in $\mo(G_2)$
(i.e., fails to rank it first).
Considering $1<k\ll m$ 
enables us to model the practical loss associated with using a VN scheme to search for $\mo(v^*)$ in $\mo(V_2)$ given limited resources.

\begin{remark}
\emph{Unlike in the classification setting described in Section~\ref{sec:BE}---where $L_F(h_n)$ is a random variable indexed by $n$---the nomination errors 
defined in Definition \ref{def:lossfcn} are constants indexed by $n$ and $m$.
In the classical setting, $L_F(h_n)$ denotes the error rate of a classifier
that classifies a single observation $X$ based on
$n$ training instances $\{(X_i,Y_i)\}_{i=1}^n$.
In the case of VN, the notion of labeled training instances is, at best, more hazy.
Indeed, in the present setting, the training data and test data are
inseparable. The graphs (or, more specifically, their edges)
\emph{are} the training data, and in the present work,
the graph orders $n,m$ are better thought of as measuring problem dimension
rather than training set size.}
\end{remark}

Analogous to the classification literature, we are now able to define the concept of Bayes optimality in the VN framework.

\begin{definition}[Bayes error of a VN scheme]
\label{def:levelkerror}
Let $(G_1,G_2)\sim F_{c,n,m,\theta}$ with vertex of interest $\vstar \in C$, and let $\mo\in\mathfrak{O}_W$ be an obfuscating function.
For $k\in[m-1]$, we define the \emph{level-$k$ Bayes optimal} VN scheme to be any element
$\Psi \in \arg \min_{\Phi \in \calVnm} L_k( \Phi, \vstar ), $
and define the \emph{level-$k$ Bayes error}
to be $\Lstar_k(\vstar) = L_k( \Psi, \vstar )$ for level-$k$ Bayes optimal $\Psi$.
\end{definition}

Now that we have a notion of Bayes error for VN, it is natural to ask whether
an optimal VN scheme exists analogous to the Bayes optimal classifier of Equation~\eqref{eq:classbe}.
Toward this end, let $(g_1,g_2)$ be realized from $(G_1,G_2)\sim F_{c,n,m,\theta}\in\mathcal{N}$, and consider a vertex of interest $v^*\in C$
and obfuscating function $\mo:V_2\rightarrow W$.
In order to avoid the technical complexities associated with graph automorphisms, in what follows we will assume that $F_{c,n,m,\theta}\in\mathcal{N}$ is supported on $\gn^a\times\gm^a$, where $\gn^a$ (resp., $\gm^a$) is the set of asymmetric graphs in $\gn$ (resp., $\gm$).
For analogous results in networks with symmetries, see Remark \ref{rem:symm}.

Letting $\simeq$ denote graph isomorphism, define the set
\begin{equation} \label{eq:recovered}
\begin{aligned}
(g_1,[\mo(g_2)])
&= \left\{\big(g_1, \tilde g_2\big)\in\gn\times\gm\text{ s.t. }\mo(\tilde g_2)\simeq\mo(g_2)\right\} \\
&= \left\{\big(g_1, \tilde g_2\big)\in\gn\times\gm\text{ s.t. }\tilde g_2\simeq g_2\right\}.
\end{aligned}
\end{equation}
In order to define the Bayes optimal scheme, we will also need the following restriction of $(g_1,[\mo(g_2)])$: for each $w\in W,$ we define
\begin{equation} \label{eq:reveal}
\begin{aligned}
(g_1,&[\mo(g_2)])_{w=\mo(v^*)}\\
&=\Big\{\big(g_1, \tilde g_2\big)\in\gn\times\gm\text{ s.t. }\exists\text{ isomorphism }\sigma\text{ s.t. }\mo(\tilde g_2)=\sigma(\mo(g_2)), \sigma(w)=\mo(v^*)\Big\}.
\end{aligned}
\end{equation}
We are now ready to define a Bayes optimal VN scheme.

For ease of notation, in the sequel we will write $\PFt$ or even simply $\p$
in place of $\PFtau$ where there is no risk of ambiguity.
Let 
\begin{equation}
\label{eq:g}
\mathbf{g}=\left\{\left(g_1^{(i)},g_2^{(i)}\right)\right\}_{i=1}^k
\end{equation} 
be such that the sets 
$$\left\{\left(g_1^{(i)},[\mo(g_2^{(i)})]\right)\right\}_{i=1}^k
$$
 partition $\gn^a\times\gm^a$.
We will call this partition $\mathcal{P}_{n,m}$, where we suppress dependence on
$\mathbf{g}$ and $\mo$ for ease of notation.
We will define a Bayes optimal scheme $\BayesVN$ (independent of the choice of $\bf g$) piecewise on each element of this partition, and we will prove in Theorem \ref{lem:crossgraph:optimal} that $\BayesVN$ is level-$k$ Bayes optimal for all $k \in [m-1]$:
\begin{equation} \label{eq:bo} 
\begin{aligned}
&\BayesVN \left(g_1^{(i)},\mo(g_2^{(i)}),v^*\right)[1] \in \underset{u \in W}{\am}\,\,
     \p \Big[ \left(g_1^{(i)},[\mo(g_2^{(i)})]\right)_{u=\mo(v^*)} \,\Big|\, \left(g_1^{(i)},[\mo(g_2^{(i)})]\right) \Big] \\
&\BayesVN\left(g_1^{(i)},\mo(g_2^{(i)}),v^*\right)[2] 
\in \underset{u \in W\setminus \{ \BayesVN[1] \} }{\am}
	\p \Big[ \left(g_1^{(i)},[\mo(g_2^{(i)})]\right)_{u=\mo(v^*)} \,\Big|\, \left(g_1^{(i)},[\mo(g_2^{(i)})]\right)\Big] \\
&\hspace{40mm}\vdots \\
&\BayesVN\left(g_1^{(i)},\mo(g_2^{(i)}),v^*\right)[m] 
\in \underset{u \in W\setminus\{\cup_{i\in[m-1]}\Phi^*[i]\} }{\am}
                        \p \Big[ \left(g_1^{(i)},[\mo(g_2^{(i)})]\right)_{u=\mo(v^*)} \,\Big|\, \left(g_1^{(i)},[\mo(g_2^{(i)})]\right) \Big],
\end{aligned}
\end{equation}
with ties broken arbitrarily but deterministically.
We refer the interested reader to Appendix \ref{sec:VNties} for discussion of the case where ties are allowed in the ranking function.
For each element 
$$(g_1,g_2)\in\left(g_1^{(i)},[\mo(g_2^{(i)})]\right)\setminus\left\{\left(g_1^{(i)},g_2^{(i)}\right)\right\},$$
choose the permutation $\sigma$ such that $\mo(g_2)=\sigma(\mo(g_2^{(i)}))$, and define
$$\BayesVN(g_1^{(i)},\mo(g_2),v^*)=\sigma(\BayesVN(g_1^{(i)},\mo(g_2^{(i)}),v^*)).$$

Lastly, the following theorem shows that this scheme (uniquely defined up to tie-breaking) is indeed Bayes optimal.
A proof can be found in Appendix \ref{sec:apx:pfoptimal}.

\begin{theorem} 
\label{lem:crossgraph:optimal}
Let $\mathfrak{o}\in\mathfrak{O}_W$ be an obfuscating function, and let
$$\mathbf{g}=\left\{\left(g_1^{(i)},g_2^{(i)}\right)\right\}_{i=1}^k$$ be such that the sets 
$$\left\{\left(g_1^{(i)},[\mo(g_2^{(i)})]\right)\right\}_{i=1}^k$$
 partition $\gn^a\times\gm^a$.
Let $\BayesVN=\BayesVN_\mathbf{g}$ be as defined in Equation~\eqref{eq:bo}.
Suppose that $(G_1,G_2)\sim F_{c,n,m,\theta}\in\mathcal{N}$ with $F_{c,n,m,\theta}$ supported on $\gn^a\times\gm^a$, and consider a vertex of interest $v^*\in C$.
We have that $L_k(\BayesVN_{\bf g},v^*)=\Lstar_k(\vstar)$ for all $k\in[m-1]$, partitions $\bf g$, and all obfuscating functions $\mathfrak{o}$.
\end{theorem}

\begin{remark}
\label{rem:symm}
\emph{
The effect of symmetries on Theorem \ref{thm:kconsis} is both subtle and cumbersome, as the specific tie-breaking procedures used in the analogue of Eq.~\eqref{eq:bo} is of great import.
To this end, consider $\bf g$ to be defined as above, and let $T\in\mathcal{T}_W$ be the ordering that specifies the (fixed but otherwise arbitrary) scheme by
which elements within each $\mathcal{I}(v;\mo(g_2))$ are ordered. 
Informally, we will first rank the sets $\mathcal{I}(v;\mo(g_2))$ (rather than the individual vertices), and then use $T$ to rank within and across each of the $\mathcal{I}(v;\mo(g_2))$. Full detail is provided below.}

\emph{For each $w\in W$ and $v\in V(g_2)$, define 
\begin{equation} \label{eq:reveal2}
\begin{aligned}
&(g_1,[\mo(g_2)])_{\mathcal{I}(w;\mo(g_2))=\mo(v)}
=\Big\{\big(g_1, \tilde g_2\big)\in\gn\times\gm\text{ s.t. }\exists\text{ iso. }\sigma\text{ with }\mo(\tilde g_2)=\sigma(\mo(g_2)),\\
&\hspace{55mm}\text{ and }\sigma(u)=\mo(v)\text{ for some }u\in\mathcal{I}(w;\mo(g_2))\Big\}.
\end{aligned}
\end{equation}
As above, we will define the Bayes' optimal VN scheme on each element of the partition provided via $\bf g$.
We first define a ranking $\Psi$ of the sets 
$$\left\{\left(g_1^{(i)},[\mo(g_2^{(i)})]\right)_{\mathcal{I}(u;\mo(g_2^{(i)}))=\mo(v^*)}\right\}_u,$$ and then will use $T$ to give the total ordering from $\Psi$.
To wit, for each $i\in[k]$, define (where ties in the $\am$ are broken in an arbitrary but nonrandom manner), iteratively define
\begin{equation} \label{eq:bos} 
\begin{aligned}
&{\Psi} \!\!\left(g_1^{(i)},\mo(g_2^{(i)}),v^*\right)\![1]\! \in \underset{\mathcal{I}(u;\mo(g_2^{(i)})) \subset  W}{\am}\,\,
\p\! \left[ \left(g_1^{(i)},[\mo(g_2^{(i)})]\right)_{\mathcal{I}(u;\mo(g_2^{(i)}))=\mo(v^*)} \bigg| \left(g_1^{(i)},[\mo(g_2^{(i)})]\right) \right] \\
&{\Psi} \!\!\left(g_1^{(i)},\mo(g_2^{(i)}),v^*\right)\![2]\!
\in \!\!\!\!\!\!\!\!\underset{\mathcal{I}(u;\mo(g_2^{(i)})) \subset W\setminus \{ {\Psi}[1] \} }{\am}
\p\! \left[ \left(g_1^{(i)},[\mo(g_2^{(i)})]\right)_{\mathcal{I}(u;\mo(g_2^{(i)}))=\mo(v^*)} \bigg| \left(g_1^{(i)},[\mo(g_2^{(i)})]\right) \right] \\
&\hspace{50mm}\vdots \\
&{\Psi} \!\!\left(g_1^{(i)},\mo(g_2^{(i)}),v^*\right)\![k]\!
\in \!\!\!\!\!\!\!\!\underset{\mathcal{I}_{\mo\left(g_2\right)}(u) \subset W\setminus\{\cup_{i=1}^{k-1}{\Psi}[i]\} }{\am}
\p\! \left[ \left(g_1^{(i)},[\mo(g_2^{(i)})]\right)_{\mathcal{I}(u;\mo(g_2^{(i)}))=\mo(v^*)} \bigg| \left(g_1^{(i)},[\mo(g_2^{(i)})]\right) \right].
\end{aligned}
\end{equation}
For each element 
$$(g_1,g_2)\in( g^{(i)}_1,[\mo(\tilde g_2^{(i)})])\setminus\{(g_1^{(i)},g_2^{(i)})\},$$ choose an isomorphism $\sigma$ such that $\mo(g_2)=\sigma(\mo(g_2^{(i)}))$, and define
$${\Psi}(g_1,\mo(g_2'),v^*)=\sigma({\Psi}(g_1^{(i)},\mo(g_2^{(i)}),v^*)).$$
Note that the choice of isomorphism $\sigma$ does not impact the definition of ${\Psi}$.}

\emph{For each $(g_1,g_2)\in\gn\times\gm$, we define a VN scheme $\BayesVN_T$ from $\Psi$ as follows:\\
{\bf 1.} Initialize $\BayesVN_T(g_1,\mo(g_2),v^*)$ as an empty list; initialize $j=1$;\\
{\bf 2.} If $\Psi(g_1,\mo(g_2),v^*)[j]$ is nonempty, add the top ranked (according to $T$) element from $\Psi(g_1,\mo(g_2),v^*)[j]$ to the end of $\BayesVN_T(g_1,\mo(g_2),v^*)$, else do nothing; set $j=j+1$ (mod $|\Psi(g_1,\mo(g_2),v^*)|$)\\
{\bf 3.} Repeat Step 2 until there are no more vertices to add to $\BayesVN_T(g_1,\mo(g_2),v^*)$.\\
If $T[1]=\mo(v^*)$, then $\BayesVN_T(g_1,\mo(g_2),v^*)$ is Bayes optimal (as in Theorem \ref{lem:crossgraph:optimal}) in the sense of Definition \ref{def:levelkerror}.
See Appendix \ref{sec:apx:pfoptimal} for details.
}
\end{remark}

\noindent{\bf Example~\ref{ex:RERP}, continued.}
Let $(G_1,G_2)\sim R$-ER($P$) for $P,R \in \mathbb{R}^{n \times n}$.
Under mild model assumptions, we have that $\lim_{n\rightarrow\infty}L^*_{k}(v^*)=0$ for any fixed $k$.
This is due to the fact that the optimal graph matching of $G_1$ to $\mo(G_2)$ will almost surely recover the true vertex labels of $\mo(G_2)$
for $n$ suitably large; i.e.,
$$\text{argmin}_{Q\in\Pi_n}\|AQ-QB\|_F=\{I_n\}\text{ with probability}\rightarrow 1,$$
where $\Pi_n$ is the set of $n\times n$ permutation matrices, $A$ is the adjacency matrix for $G_1$ and $B$ the adjacency matrix for $G_2$.
More concretely, we have the following theorem adapted to our present setting from \cite{unmatchable}.  A proof sketch can be found in Appendix \ref{sec:gmpf}.
\begin{theorem}
\label{thm:stochmatchedcores}
Let $(A,B)\sim R$-ER($P$), and for any fixed permutation matrix $Q$ define the random variable $\delta(Q):= \|AQ-QB\|_F$.
Define $0<\epsilon:=\min_{i,j; i\neq j}2R_{i,j}P_{i,j}(1-P_{i,j})$.
There exist positive constants $c_1,c_2$ such that if $\epsilon^2>c_1 \log(n)$,
then for sufficiently large $n$, 
$$\mathbb{P}\left(\exists\,\,Q\in \Pi_n\setminus\{I_n\}\text{ s.t. } \delta(Q)\leq \delta(I_n)\right)\leq2\exp \left\{ -c_2 \epsilon^2n \right\}.$$
\end{theorem}
\noindent 
Similarly to the Bayes optimal scheme, we define the graph matching VN scheme, denoted $\Phi_M$, separately on each element of $\mathcal{P}_{n,n}$.
For a given $(\tilde g_1,[\mo(\tilde g_2)])\in\mathcal{P}_{n,n},$ let $(g_1,g_2)$ be an fixed element in $(\tilde g_1,[\mo(\tilde g_2)])$.
Define $\Phi_M(g_1,\mo(g_2),v^*)[1]$ to be a fixed but arbitrary element from
\begin{equation} \label{eq:def:VNviaGM}
\left\{{\mathfrak r}^{-1}(v^*) \text{ s.t. } Q_{\mathfrak r}\in \text{argmin}_{Q\in\Pi_n}\left\|AQ-QB\right\|_F\right\},
\end{equation}
where each $\mathfrak r:W\rightarrow V_2$ appearing above is a bijection and $Q_{\mathfrak r}$ its associated permutation matrix (having identified both $W$ and $V_2$ with the set $[n]$).
Define
$$
\mathfrak R_1=\Big\{\mathfrak r:W\rightarrow V_2  \text{ s.t. }\mathfrak{r}\text{ is a bijection}  \text{ with }\mathfrak r^{-1}(v^*)=\Phi_M(g_1,\mo(g_2),v^*)[1]\Big\}.
$$
If $i>1$, define
$\Phi_M(g_1,\mo(g_2),v^*)[i]$ to be any element of
$$
\Big\{ {\mathfrak r}^{-1}(v^*)\text{ s.t. }Q_{\mathfrak r}\in
\underset{Q\in\Pi(n)\setminus\{Q_{\mathfrak r} : \mathfrak{r}\in\cup_{j=1}^{i-1}\mathfrak R_j\}}{\text{argmax}} \left\|AQ-QB\right\|_F\Big\},
$$
where $\mathfrak R_j$ is defined analogously to $\mathfrak R_1$.
For each element $(g_1',g_2')\in(\tilde g_1,[\mo(\tilde g_2)])\setminus\{(g_1,g_2)\}$, choose a permutation $\sigma$ such that $\mo(g_2')=\sigma(\mo(g_2))$, and we then define
$\Phi_M(g_1,\mo(g_2'),v^*)=\sigma(\Phi_M(g_1,\mo(g_2),v^*)).$
Theorem \ref{thm:stochmatchedcores} states that under mild model assumptions, we have that $\Phi_M(G_1,\mo(G_2),v^*)[1]=\mo(v^*)$ asymptotically almost surely, and thus
$\lim_{n\rightarrow\infty}L_1(\Phi_M,v^*)=0$.
Indeed, in this setting, for any fixed $k \ge 1$,
$\lim_{n\rightarrow\infty}L_k(\Phi_M,v^*)=0$.
It is then immediate that $\lim_{n\rightarrow\infty}L^*_{k}(v^*)=0$ in this model for any fixed $k$ as desired.
\\

The next two examples serve to illustrate how the level-$k$ Bayes error behaves
in the presence of stochastically indistinguishable vertices.
In essence, we cannot hope to perform better than randomly ordering
stochastically equivalent vertices.

\vspace{3mm}

\noindent{\bf Example \ref{ex:indERgraphs}, continued.}
Let $G_1$ and $G_2$ be independent $\ER(n,p)$ graphs.
Since the vertices are stochastically indistinguishable within each of the two graphs, no nomination scheme can do better than random chance in this model.
Thus, with $c=n$, we have that
$L^*_k(v^* ) =(1- k/n)(1+o(1))$ for all $k\in[n-1]$ and all $v^*\in[n]$.

\begin{example}
\label{ex:SBM}\emph{
Let $p_1,p_2,q \in [0,1]$ with $1 \ge p_1>p_2 \ge 0$ and $q\neq p_1,p_2$.
Define the matrix
$$ B = \begin{pmatrix} p_1 & q\\q& p_2\end{pmatrix}, $$
and let $G_1$ and $G_2$ be independent $\SBM\left(2,B,b_n\right)$ graphs where
$b_n(i)=1$ if $i\leq n$ and $b_n(i)=2$ if $i>n$.
With $c=n$ and the correspondence equal to the identity function, let $(k_n)_{n=2}^\infty$ be a nondecreasing divergent sequence satisfying $k_n\leq n$ for all $n>1$,
then $\lim_n L^*_{k_n}(v^* ) =\lim_n \left[(1- 2k_n/n)\vee 0\right]$ for all $v^*$.
Indeed, $L^*_{k_n}$ is asymptotically equivalent to randomly ordering the $n/2$
vertices in $G_2$ that are stochastically equivalent to $v^*$.}
\end{example}

\section{VN consistency}
\label{sec:VNCon}

With the definition of Bayes optimality and the Bayes optimal scheme in hand,
it is now possible to define a notion of consistent vertex nomination
analogous to consistent classification in the pattern recognition literature.
Before defining a consistent VN rule (i.e., a sequence of VN schemes), we must first define the notion of sequences of distributions in $\mathcal{N}$ with {\em nested cores}. Such sequences of distributions are necessary in order to speak sensibly of a sequence of vertex nomination problem instances.
\begin{definition}
Let $\mathbf{F}=\left(F_{c(n),n,m(n),\theta(n)}\right)_{n=n_0}^\infty$ be a sequence of distributions in $\mathcal{N}$.
We say that $\bf F$ has {\em nested cores} if there exists an $n_0$ such that for all $n_0\leq n<\tilde n$, if $(G_1,G_2)\sim F_{c(n),n,m(n),\theta(n)}$ and $(\widetilde G_1,\widetilde G_2)\sim F_{c(\tilde n),\tilde n,m(\tilde n),\theta(\tilde n)}$, we have,
letting $C$ and $\tilde C$ be the core vertices
associated with $F_{c(n),n,m(n),\theta(n)}$
and $F_{c(\tilde n),\tilde n,m(\tilde n),\theta(\tilde n)}$ respectively,
and denoting the junk vertices $J_1,\widetilde J_1,J_2,\widetilde J_2$
analogously,\\
$[i.]$ $V(G_1)=C\cup J_1\subset V(\widetilde G_1)=\widetilde C\cup \widetilde J_1$;\\
$[ii.]$ $V(G_2)=C\cup J_2\subset V(\widetilde G_2)=\widetilde C\cup \widetilde J_2$;\\
$[iii.]$ $C\subset \widetilde C$.
\end{definition}
\noindent We are now ready to define a consistent VN rule.
\begin{definition}
\label{def:consis}
Let $\mathbf{F}=\left(F_{c(n),n,m(n),\theta(n)}\right)_{n=n_0}^\infty$ be a sequence of nominatable distributions in $\mathcal{N}$ with nested cores satisfying
$\lim_{n \rightarrow \infty} m(n)=\infty$.
For a given non-decreasing sequence $( k_n )$, we say that a VN rule $\bp=(\Phi_{n,m(n)})_{n=n_0}^\infty$ is \emph{level-$(k_n)$ consistent} at $v^*$ with respect to $\mathbf{F}$ if
$$ \lim_{n\rightarrow\infty}L_{k_n}( \Phi_{n,m(n)}, \vstar )-\Lstar_{k_n}(\vstar )=0,$$ for any sequence of obfuscating functions of $V_2$ with $|V_2|=m(n)$.
If a rule $\bp$ is level-$(k_n)$ consistent at $v^*$ for
a constant sequence $k_n = k$,
$n=1,2,\dots$, then we say simply that
$\bp$ is \emph{level-$k$ consistent}.
\end{definition}

\begin{remark}
\emph{
Equation~\eqref{eq:Lchain} has an interesting implication for VN consistency in the setting where $\Lstar_{k_n}( \vstar ) \rightarrow 0$.
In this case, level-$(k_n)$ consistency of a VN rule $\bp$ implies that $\bp$ is $(k'_n)$-consistent for all $(k'_n)$ such that $\liminf{\frac{k'_n}{k_n}}\geq 1$.
We conjecture that this implication holds true for the case where $\Lstar_{k_n}( \vstar ) \rightarrow c>0$, but this problem remains open at present.}
\end{remark}

\noindent{\bf Example \ref{ex:RERP}, continued.}
Let $\mathbf{F}=(F_{n,n,n,\theta_n=(P_n,R_n)})$ be a sequence of $R_n$-$\ER(P_n)$
random graph models in $\mathcal{N}$ for some sequence of probability matrices
$(P_n)_{n=n_0}^\infty$ and correlation matrices $(R_n)_{n=n_0}^\infty$.
Under mild model assumptions (see Theorem \ref{thm:stochmatchedcores}), the graph matching vertex nomination rule $\bp_M$ defined in Equation~\eqref{eq:def:VNviaGM} above is level-$1$ consistent, and hence level-$(k_n)$ consistent for all $(k_n)$ sequences.
\\

\noindent{\bf Example \ref{ex:indERgraphs}, continued.}
Let $\mathbf{F}=(F_{n,n,n,\theta_n=p})$ be a sequence of independent $\ER(n,p)$
random graph models in $\mathcal{N}$.
All vertex nomination rules are level-$(k_n)$ consistent for all $(k_n)$ sequences.
This holds for all possible values of $c\in[n]$ in the nested sequence of $\ER(n,p)$ distributions, as all VN rules have effectively chance performance, regardless of core size under this model.
\\

We define the consistency of a VN rule with respect to a broad class of graph sequences, and it is perhaps no surprise that there cannot be any level-$(k_n)$ universally consistent VN rules, not even for constant sequences $k_n:=k$ (i.e., those that are level-$(k_n)$ consistent for all sequences of nominatable distributions $\bf F$ with nested cores).
To prove this result, we will first establish an analogue to the ``arbitrary poor performance'' theorems for classifiers, see Theorem 7.1 of \cite{DGL} which state that for a fixed $n$ and $m$, any VN scheme can be shown to have arbitrarily poor performance with respect to a well-chosen adversarial distribution $F_{c,n,m,\theta}$.
Our theorem mirrors the classical classification literature, as for a given classification rule,
there exists ``a sufficiently complex distribution for which the sample size $n$ is hopelessly small,'' \cite{DGL} pg. 111,
so that a classification rule can be made to perform arbitrarily poorly by selecting a suitably complex data distribution.
Nonetheless, in the case of classification, this model complexity and the implicit dependence on $n$ can be overcome asymptotically by a classification rule. That is, universally consistent classifiers exist; see, for example, \cite{Stone1977,steinwart2002support,tang2013universally}.
In contrast, in the VN problem, the complexity of the model generating the data can also grow in $n$, which effectively thwarts the ability of a VN rule to asymptotically overcome a sequence of adversarial graph models.
Formalizing the above, we arrive at the following theorem, a proof of which can be found in Appendix~\ref{sec:apx:kconsis}.

\begin{theorem}
\label{thm:kconsis}
Let $n$ and $m$ be large enough to guarantee the existence of asymmetric graphs $g_1\in\gn$ and $g_2\in\gm$.
Consider a VN scheme $\Phi\in\mathcal{V}_{n,m}$, obfuscating function $\mo$, and strictly increasing sequence $(\epsilon_i)_{i=1}^m$ satisfying
$\epsilon_i\in(0,\frac{i}{m})$.
Then
there exists a distribution $F_{c,n,m,\theta}\in\mathcal{N}$ over $\calGn \times \gm$
and $v^*\in C$ such that for each $k\in [c-1]$,
$$ \Lstar_k(v^*)\leq\epsilon_{m-k} < 1-\frac{k}{m}< 1-\epsilon_k<L_k( \Phi,v^* ), $$
where $1-\frac{k}{m}$ represents the error probability of chance performance; i.e., the error probability of a VN scheme in the independent Erd\H os-R\'enyi setting.
\end{theorem}

In the remainder of the section, we will suppress the dependence of $m=m(n)$ on $n$.
If we consider sequences $(\epsilon_{m,i})_{i=1}^{m}$ satisfying the assumptions of Theorem \ref{thm:kconsis} and $\lim_n\epsilon_{m,m-k_n}=\epsilon\in(0,1)$ for a given $(k_n)$ satisfying $k_n=o(m)$, we arrive at the following Corollary, namely that level-$(k_n)$ universally consistent VN schemes do not exist for any sequence $(k_n)_{n=n_0}^\infty$ that does not grow as fast as $m = |V(G_2)|$.
\begin{corollary}
\label{cor:univconsis}
Let $\epsilon\in(0,1)$ be arbitrary, and consider a VN rule $\boldsymbol{\Phi}=(\Phi_{n,m})$.
For any nondecreasing sequence $(k_n)_{n=n_0}^\infty$ satisfying $k_n=o(m)$, there exists a sequence of distributions $(F_{c,n,m,\theta})$ in $\mathcal{N}$ with nested cores such that
$$ \lim \sup_{n\rightarrow\infty}\Lstar_{k_n}(v^*)=\epsilon < 1=\lim_{n\rightarrow\infty}L_{k_n}( \Phi_{n,m},v^* ).$$
\end{corollary}

\noindent Corollary~\ref{cor:univconsis} has a number of practical implications.
Below, we will briefly outline two such implications.
Unlike in the classification setting, where universally consistent rules (e.g., $k$-nearest neighbors) are theoretically guaranteed to perform well in big-data settings, the VN practitioner enjoys no such certainty.
Indeed, in VN, the practitioner first needs to identify the consistency class of a VN rule (i.e., the set of models for which the VN rule is consistent) before applying it in real settings.
Unfortunately, identifying and enumerating these consistency classes is theoretically and practically nontrivial, and we are investigating theory and heuristics for this at present.
In a streaming data setting, the performance of a universally consistent classifier will approach Bayes optimality for the distribution governing the data, and the classifier will be guaranteed to successfully adapt itself to any changes in the underlying data distribution.  
The lack of universal consistency in the VN setting implies that this is not the case, as the performance of a consistent VN scheme in the streaming setting could precipitously decline in the presence of distributional shifts in the data.
Recognizing these shifts and their potential impact on VN performance is paramount and is the subject of current research.

\subsection{Global consistency}
We have just seen that no universally consistent VN schemes exist.
This is a consequence of the complexity of the models available
when choosing a sequence of nominatable distributions.
Indeed, nested-core nominatable sequences $\bf F$ allow for (nearly) arbitrary dependence structure and model complexity
as $n$ increases: corresponding vertex behavior may be correlated (see Example \ref{ex:RERP}), independent (see Example \ref{ex:SBM}), or negatively correlated (see Example \ref{ex:behaviorchange}) across networks.
This model flexibility is in service of modeling the complexity of real world networks, but, as we will demonstrate below, restricting our model class to simpler dependency structures still does not necessarily guarantee the existence of universally consistent schemes.

It is thus natural to explore a weaker notion of consistency, namely consistency for a sufficiently large family of nominatable sequences rather than for all nominatable sequences.
\begin{definition}
Let $\mathfrak{F} = \{ \mathbf{F}_\alpha = (F^{\alpha}_{c(n),n,m(n),\theta(n)})_{n=n_0}^\infty : \alpha \in \mathcal A\}$
be a family of nominatable sequences,
indexed by some set $\mathcal A$.
We say that VN scheme $\Phi$ is
\emph{level-$(k_n)$ $\mathfrak{F}$-globally consistent} 
if $\Phi$ is level-$(k_n)$ consistent
for every $\bf F \in \mathfrak F$.
We call such a family \emph{level-$(k_n)$ globally consistent}.
\end{definition}
\noindent The question of the maximal family $\mathfrak{F}$ for which a level-$(k_n)$ $\mathfrak{F}$-globally consistent rule exists is of prime interest. While we cannot offer a satisfactory complete answer to that question in the present work, we do offer some examples of jointly consistent families.

\vspace{2mm}
\noindent {\bf Example \ref{ex:RERP} continued:}
In settings where corresponding vertices have correlated neighborhood structures across networks, there is hope for finding globally consistent rules.
In the ongoing Example \ref{ex:RERP}, we have seen a simple example of this in the $R$-$\ER(P)$ model, in which the matrix of correlations $R$ encodes a correspondence across the two graphs.
As mentioned previously, Theorem \ref{thm:stochmatchedcores} asserts that under some mild model assumptions on $R$ and $P$ in the $R$-$\ER(P)$ model, level-1 globally consistent VN rules exist (namely the graph matching VN scheme).
If $\mathfrak F$ denotes the set of distributions obeying these model assumptions, then we have that level-$(k_n)$ $\mathfrak F$-globally consistent rules exist for all sequences $(k_n)$.
While we do not expect the conditions of Theorem \ref{thm:stochmatchedcores} to produce a maximal level-$(k_n)$ globally consistent family for any given $(k_n)$, this example nonetheless provides an important intuition for the properties such maximal families might possess.

\vspace{2mm}
\noindent {\bf Example \ref{ex:SBM} continued:}
The SBM provides a prime example of global consistency.
Working in the one-graph framework of Remark~\ref{rem:one-graph},
under appropriate growth conditions on the parameters
of every sequence in family $\mathfrak{F}$,
Theorem 6 in \cite{lyzinski2016consistency} implies the existence of a
likelihood-based nomination scheme that is
level-$(|U_1|)$ globally consistent for this family of models.
Under similar growth conditions, Theorem 6 in \cite{perfect} implies the
existence of a level-$(|U_1|)$ globally consistent scheme based on
spectral clustering, in which vertices are nominated based on their proximity
to the vertex or vertices of interest.

\begin{remark}
\emph{An attempt at systematically constructing a maximal globally consistent family might begin by putting model restrictions onto elements of $\mathcal{N}$.
A natural restriction to consider would be to demand that the models in $\mathfrak{F}$ be nested in the following sense:
For $\bf F\in\mathfrak{F}$, if
$(G_1,G_2)\sim F_{c(n_2),n_2,m(n_2),\theta(n_2)}$
with $n_1<n_2$, then
$(G_1\big[ [n_1]\big],G_2\big[ [m(n_1)]\big])\stackrel{\mathcal{L}}{=}(G_1',G_2')$ where
$(G_1',G_2')\sim F_{c(n_1),n_1,m(n_1),\theta(n_1)}$.
This property would allow us to consider ``streaming'' network models $\bf F$, where for $n_1<n_2$, if $(g_1,g_2)$ is realized from $(G_1,G_2)\sim F_{c(n_2),n_2,m(n_2),\theta(n_2)}$, and $(g_1',g_2')$ is realized from $(G_1',G_2')\sim F_{c(n_1),n_1,m(n_1),\theta(n_1)}$ then $(g_1,g_2)$ can be constructed by appropriately adding $n_2-n_1$ vertices to $(g_1',g_2')$. 
Additionally, this would serve to mimic the nested nature of the data in the classification consistency literature.
However, as we will see in Example \ref{ex:behaviorchange}, global consistency
depends both on the dependency structure within each graph (as seen in Theorem \ref{thm:kconsis}) \emph{and} the vertex correspondence (i.e., the potential dependency structure across graphs) encoded in the model.}
\end{remark}

\subsection{Behavioral (in)consistency and global (in)consistency }
We suspect that if the vertices of interest have a common distinguishing probabilistic and/or topological characteristic (e.g., correlated neighborhoods, common SBM block structure, high network centrality, etc.) then a globally consistent rule may exist.
Indeed, under mild model assumptions, this is the case in the $\RER(P)$ of Example~\ref{ex:RERP}; in the i.i.d.\ SBM of Example~\ref{ex:SBM} where the correspondence is the identity function \cite{perfect}; and in the i.i.d.\ ER of Example~\ref{ex:indERgraphs}, to name a few.
In each of these examples, there is a stochastic/topological similarity (or in the ER case, uniformity) between corresponding vertices across networks. In each, corresponding vertices behave similarly across networks.
While we suspect that this behavioral similarity is not sufficient for global consistency, Example~\ref{ex:behaviorchange} demonstrates that behavioral inconsistencies within a family of nominatable distributions can preclude the existence of globally consistent nomination rules.



\begin{example} \label{ex:behaviorchange}
\emph{
For each $n$, consider $n$-vertex random graphs $G_1{\sim}\asymSBM(2,B_1,b^{(1)}_n)$ independent of $G_2{\sim}\asymSBM(2,B_2,b^{(2)}_n)$, where $\asymSBM$ denotes the stochastic blockmodel distribution restricted to have support on asymmetric graphs.  
This restriction is made to avoid the unpleasantries of symmetries, and is asymptotically negligible as the SBM's considered in this example are asymptotically almost surely asymmetric. \\
{\bf Case 1.} In this case, corresponding vertices behave similarly across networks.
To wit, let ${\bf F}=(F_n)_{n=n_0}^\infty$ be the sequence of models where
$$
B_1=B_2=\begin{pmatrix}p_1&q\\q&p_2  \end{pmatrix},\,\,
b_n^{(1)}(i)=b_n^{(2)}(i)=\begin{cases}1&\text{ if }i\leq n/2;\\
2&\text{ if }i> n/2,\end{cases}
$$
$p_1\neq p_2$, $c=n$, and the correspondence is the identity function.
As stated before, in this model $L^*_{n/2}(v^*)\rightarrow 0$ for all $v^*\in C$.
Without loss of generality, consider $v^*=v_1=u_1$.
If $\mathbf{\Phi}$ is consistent with respect to ${\bf F}$ then
$$\p_{F_n}(\rank_{\Phi(G_1,\mo(G_2),v_1)}(\mo(u_1))\geq n/2+1)\rightarrow 0.$$
By the distributional equivalence of vertices within the same block, and the consistency property in the definition of a VN scheme, for any $u,v\in b_n^{-1}(1)$, $k\in[n]$ we have that
$$
\p_{F_n}(\rank_{\Phi(G_1,\mo(G_2),v_1)}(\mo(u))=k)  =\p_{F_n}(\rank_{\Phi(G_1,\mo(G_2),v_1)}(\mo(v))=k).
$$
Letting this common value be set to $\alpha_{k,n}$ (with $\beta_{k,n}$ defined similarly as the common value of $\p_{F_n}(\rank_{\Phi(G_1,\mo(G_2),v_1)}(\mo(u))=k)$ for $u$ in block $2$), we have that
$$
\sum_{i=1}^n  \p_{F_n}(\rank_{\Phi(G_1,\mo(G_2),v_1)}(\mo(u_i))=k)=1=\frac{n(\alpha_{k,n} + \beta_{k,n})}{2}
$$
giving us that $\alpha_{k,n}=2/n-\beta_{k,n}$.
Consistency implies that
$$\sum_{k=1}^{n/2}\alpha_{k,n}\rightarrow 1,$$
which implies that
$$\sum_{k=1}^{n/2}\alpha_{k,n}=\sum_{k=1}^{n/2}(2/n-\beta_{k,n})=1-\sum_{k=1}^{n/2}\beta_{k,n}\rightarrow 1, $$
implying $\sum_{k=1}^{n/2}\beta_{k,n}\rightarrow 0.$
Therefore, for any $u$ in block $2$,
$$\p_{F_n}(\rank_{\Phi(G_1,\mo(G_2),v_1)}(\mo(u))\geq n/2+1)\rightarrow 1.$$
{\bf Case 2.}
In this case, corresponding vertices behave differently across networks.
To wit, let ${\bf \tilde F}=(\tilde F_n)_{n=n_0}^\infty$ be the sequence of models where 
$$B_1=\begin{pmatrix}p_1&q\\q&p_2  \end{pmatrix},\hspace{5mm} B_2=\begin{pmatrix}p_2&q\\q&p_1  \end{pmatrix},\,\,
b_n^{(1)}(i)=b_n^{(2)}(i)=\begin{cases}1\text{ if }i\leq n/2;\\
2\text{ if }i> n/2,\end{cases}
$$
$c=n$, and the correspondence is the identity function.
As in Case 1 considered above, in this model $L^*_{n/2}(v^*)\rightarrow 0$ for all $v^*\in C$, and, as above, consider $v^*=v_1=u_1$.
If $\mathbf{\Phi}$ is consistent with respect to ${\bf \tilde F}$ then
$$\p_{\tilde F_n}(\rank_{\Phi(G_1,\mo(G_2),v_1)}(\mo(u_1))\geq n/2+1)\rightarrow 0.$$}

\emph{
Note that if $\sigma$ is the permutation such that
$$\sigma(i)=\begin{cases}i+n/2&\text{ if }i\leq n/2;\\
i-n/2&\text{ if }i> n/2,\end{cases},$$
then $\p_{F_n}(g_1,g_2)=\p_{\tilde F_n}(g_1,\sigma(g_2))$.
Define
$$E_n=\{(g_1,g_2)\text{ s.t. }\rank_{\Phi(g_1,\mo(g_2),v_1)}(\mo(u_{1}))\geq n/2+1\}\}$$ i.e.,
and
$\tilde E_n=\{(g_1,g_2)\text{ s.t. }(g_1,\sigma(g_2))\in E_n\}$, i.e.,
$$\tilde E_n=\{(g_1,g_2)\text{ s.t. }\rank_{\Phi(g_1,\mo(g_2),v_1)}(\mo(u_{n/2+1}))\geq n/2+1\}.$$
If $\mathbf{\Phi}$ is consistent with respect to $\bf F$ we have that $\p_{F_n}(E_n)\rightarrow 0$ which implies (as $(g_1,g_2)\in E_n\Leftrightarrow (g_1,\sigma(g_2))\in \tilde E_n$)
$\p_{\tilde F_n}(\tilde E_n)\rightarrow 0$.
If $\mathbf{\Phi}$ is consistent with respect to $\bf \tilde F$ then $\p_{\tilde F_n}(E_n)\rightarrow 0$
and
$\p_{\tilde F_n}(\tilde E_n)\rightarrow 1$.
We arrive at a contradiction, and $\mathbf{\Phi}$ cannot be $(n/2)$-consistent with respect to both $\bf F$ and $\bf \tilde F$.}
\end{example}

Although Example \ref{ex:behaviorchange} may seem artificial, it is a simple representation of a common phenomenon observed in network data.
Often the same entity can behave quite differently across networks (see \cite{patsolic2017vertex} for an example of this in social networks and \cite{jointLi} for an example of this in connectomics).
In such a setting, intuition says that a universal scheme that works in both behavioral settings should not exist. Indeed, at least in the simple block model setting considered above, we see that no such scheme exists.
Example \ref{ex:behaviorchange} also highlights an important difference between the VN setting and the more standard classification framework.
We already noted that classification's universal consistency relies on the distribution not changing in $n$, whereas in VN the distributions must vary with $n$ (indeed, the graph sizes grow in $n$).
Further, this example shows that the nonexistence of a universally consistent scheme is not simply a consequence of changing the underlying distributional parameters with $n$, as these two SBM distributions are (essentially) fixed, in that the matrix $B$ does not change with $n$.
In this example the ``training data" provided by $G_1$ cannot be uniformly beneficial for a single VN scheme across the two differing model settings we consider. In contrast, in the classification setting of \cite{DGL}, the training data uniformly provides progressively better estimates of the class-conditional distributions, whereas here it does not.
Indeed, the training data helps delineate potentially interesting vertices from non-interesting ones in $G_2$ in Case (1) for one VN scheme, and in Case (2) for another VN scheme,
but there does not exist a VN scheme that achieves this desired class separation across both cases.

\begin{remark}
\emph{
In the cases considered in Example \ref{ex:behaviorchange}, if we introduce positive edge-wise correlation of 
$$\rho\leq  \sqrt{\min\left(\frac{p_1(1-p_2) }{p_2(1-p_1)},\frac{p_2(1-p_1) }{p_1(1-p_2)}\right)}$$
into both Case 1 and Case 2, then under mild assumptions on the growth of $p_1$ and $p_2$, joint consistency can be recovered via a USVT centered graph matching nomination scheme; for details see \cite{unmatchable}.
This example demonstrates that it is sometimes possible to toggle a family of models to allow for global consistency.
A note of caution is needed, however, as in this particular example the correlation $\rho$ is introducing a behavioral consistency across networks that addresses the precise issue brought forth by the behavioral inconsistency in Example \ref{ex:behaviorchange}.  
In other, more nuanced model families, we do not expect the global-consistency modification (if it indeed exists) to be as straightforward as adding additional edge-correlation into the model.
}
\end{remark}
\subsection{Vertex nomination on networks with node covariates}
\label{sec:features}

It is natural to ask if incorporating vertex features into the VN framework can resolve the lack of universally consistent VN schemes.
While straightforward to implement, the ameliorating effect of features is significantly more nuanced.
Before defining the VN scheme with features, we need the following extension of $\mathcal{I}(v;g)$ for $g\in\gn$ and $v\in V(g)$.
Letting $\mathcal{X}$ be the space of vertex features for graphs in $\gn$, for $g\in\gn$, $v\in V(g)$, and $X\in\mathcal{X}^n$ we define 
$$\widetilde{\mathcal{I}}(v;g,X)=\{u\in V(g): \exists\text{ automorphism }\tau\text{ of }g\text{ s.t. }\tau(v)=u\text{ and }X_u=X_v\},$$
where $X_v$ is the feature associated to $v$ via $X$. 

\begin{definition}[Vertex Nomination (VN) Scheme with features]\label{def:VNf}

Let $\mathcal{X}$ (resp., $\mathcal{Y}$) be the space of vertex features of graphs in $\mathcal{G}_n$ (resp., $\mathcal{G}_m$).
For $n,m>0$ and obfuscating set $W$ fixed, a \emph{vertex nomination scheme with features} is a function 
$$\Phi: \gn \times \mathcal{X}^n \times \gm \times \mathcal{Y}^m \times \mathfrak{O}_W \times V_1 \rightarrow \mathcal{T}_{W}$$ 
satisfying the following consistency property:
If for each $u\in V_2$, we define $$\text{rank}_{\Phi(g_1,X,\mo(g_2),\mo(Y),v^*)}\big(\mo(u)\big)$$ to be the position of $\mo(u)$ in the total ordering provided by $\Phi(g_1,X,\mo(g_2),\mo(Y),v^*)$, and we define
$\mathfrak{r}_{\Phi}:\gn \times \mathcal{X}^n \times \gm \times \mathcal{Y}^m \times\mathfrak{O}_W\times V_1\times2^{V_2}\mapsto 2^{[m]}$ via
$$\mathfrak{r}_{\Phi}(g_1,X,\mo(g_2),\mo(Y),v^*,S)=\{\text{rank}_{\Phi(g_1,X,\mo(g_2),\mo(Y),v^*)}\big(\mo(u)\big)\text{ s.t. }u\in S \},$$
then we require that for any $g_1\in\gn,$ $g_2\in\gm$, $v^*\in V_1$, $X\in\mathcal{X}^n,$ $Y\in\mathcal{Y}^m,$ obfuscating functions $\mo_1,\mo_2\in\mathfrak{O}_W$ and any $u\in V(g_2)$, 
\begin{align}
\label{eq:consisfeat}
&\mathfrak{r}_{\Phi}\big(g_1,X,\mo_1(g_2),\mo_1(Y),v^*,\widetilde{\mathcal{I}}(u;g_2,Y)\big)
=\mathfrak{r}_{\Phi}\big(g_1,X,\mo_2(g_2),\mo_2(Y),v^*,\widetilde{\mathcal{I}}(u;g_2,Y)\big).
\end{align}
We let $\widetilde{\mathcal{V}}_{n,m}$ denote the set of all such VN schemes.
\end{definition}

It is immediate that if the features are sufficiently informative, consistency can be established with features where it could not be without.  
Indeed, consider in Example \ref{ex:behaviorchange} features that encode the community memberships of a few vertices (e.g., a few vertices whose correspondences across the two graphs are known a priori).
Combined with spectral methods, these would be sufficient for consistent VN under either behavior regime.
It is also immediate that the fundamental idea presented in Example \ref{ex:behaviorchange} has an analogue when vertex features are available,
as illustrated by the following example.
\begin{example} \label{ex:behaviorchangefeat}
\emph{
For each $n$, consider $n$-vertex random graphs $G_1{\sim}\asymSBM(3,B_1,b^{(1)}_n)$ independent of $G_2{\sim}\asymSBM(3,B_2,b^{(2)}_n)$, where $\asymSBM$ again indicates the stochastic block model with support restricted to the asymmetric graphs.\\
{\bf Case 1.} In this case, corresponding vertices behave similarly across networks.
To wit, let ${\bf F}=(F_n)_{n=n_0}^\infty$ be the sequence of models where $3|n$ and 
$$
B_1=B_2=\begin{pmatrix}p_1&q& q\\q&p_2& q\\q&q& p_1  \end{pmatrix},\,\,
b_n^{(1)}(i)=b_n^{(2)}(i)=\begin{cases}1&\text{ if }i\leq n/3;\\
2&\text{ if }i\in (n/3,2n/3]\\
3&\text{ if }i>n/3
,\end{cases}
$$
$p_1\neq p_2$, $c=n$, and the correspondence is the identity function.\\
{\bf Case 2.} In this case, corresponding vertices behave differently across networks.
To wit, let ${\bf F}=(F_n)_{n=n_0}^\infty$ be the sequence of models where $3|n$ and 
$$
B_1=\begin{pmatrix}p_1&q& q\\q&p_2& q\\q&q& p_1  \end{pmatrix},\,\,B_2=\begin{pmatrix}p_2&q& q\\q&p_1& q\\q&q& p_1  \end{pmatrix},\,\,
b_n^{(1)}(i)=b_n^{(2)}(i)=\begin{cases}1&\text{ if }i\leq n/3;\\
2&\text{ if }i\in (n/3,2n/3]\\
3&\text{ if }i>n/3
,\end{cases}
$$
$p_1\neq p_2$, $c=n$, and the correspondence is the identity function.\\
Similar to Example \ref{ex:behaviorchange}, without features no VN scheme can be consistent for both Cases 1 and 2.
In a similar fashion, if we consider features $X$ and $Y$ defined via
$$X_v=Y_v=\begin{cases}
1&\text{ if }b^{(1)}(v)=1\\
-1&\text{ if }b^{(1)}(v)=2,3,
\end{cases}
$$ 
then joint consistency is achievable for both Cases 1 and 2, for example by relying on features and ignoring graph structure.
However, if we consider features
$X$ and $Y$ defined via
$$X_v=Y_v=\begin{cases}
1&\text{ if }b^{(1)}(v)=1,2\\
-1&\text{ if }b^{(1)}(v)=3,
\end{cases}
$$ 
then joint consistency is again not achievable for both Cases 1 and 2.
}
\end{example}
This example demonstrates that features, in general, are not enough to ensure universal consistency.
Nevertheless, insofar as features supply additional information, they can improve VN performance.
A more thorough examination of the effect and effectiveness of vertex features in VN is beyond the scope of this work, and is the subject of current research.

\section{Discussion}
\label{sec:diss}
In this work, we have introduced a notion of consistency for the vertex nomination task that better reflects the broad range of models under which VN may be deployed.
Rather than being restricted to the stochastic block model structure required in previous formulations of the problem, our framework allows for arbitrary dependence structure both within and between graph pairs, while encompassing the original SBM formulation of the problem. Additionally, we have demonstrated how this framework relates to the well-studied notion of Bayes optimal classifiers in the pattern recognition literature.
Unlike in the classification setting, we have seen that while Bayes optimal VN schemes always exist, no universally consistent scheme exists.
This fact is due essentially to the additional leeway provided by the graph model, in which observing more vertices does not necessarily correspond to receiving more information about the underlying distribution.
This is in contrast to the classification setting studied in \cite{Stone1977} and others \cite{DGL}, in which observing more samples allows more accurate estimation of the underlying distribution and class boundary.
For this reason, one especially interesting line of investigation concerns the nominatable distributions for which larger $n$ does indeed correspond to more information about the underlying graph distribution.
A simple example of this is the initial formulation of the vertex nomination problem, in which observing more vertices allows one to better estimate the model parameters, including the block memberships, and thus more accurately identify the vertices from the interesting block.
We suspect that the essential property at play here is that under models of this sort, each vertex is analogous to a sample from a single distribution, though this may not be in and of itself a sufficient condition for consistency.
For example, in the case of $(G_1,G_2)$ being i.i.d.\ or $\rho$-correlated marginally identical draws from a random dot product graph model with the identity correspondence, each vertex (along with its incident edges) is, in a sense, a noisy sample from the underlying latent position distribution.
Hence, for large $n$, one can estimate the latent positions or their distribution to arbitrary accuracy, and provided that the latent positions of the interesting vertices are suitably separated from those of the rest of the graph, one should have VN consistency for the collection of these latent position models.

More broadly, it would be good to better understand whether there exist families of nominatable distributions $\mathfrak F$ for which certain VN schemes are consistent, and precisely how large these families can be made to be.
In a similar vein, it would be of interest to explore how the dependence structure allowed both within and between graphs influences vertex nomination.
In particular, if one rules out certain pathologically hard dependence structures as considered in Example \ref{ex:behaviorchange}, can one obtain global consistency with respect to this restricted set of distributions?
We hope to explore these questions in future work.

We are also exploring alternative formulations of the VN problem and alternate formulations of the VN loss function.
While the extension to multiple vertices of interest in each network $G_1$ and $G_2$ is straightforward, we are considering several generalizations of the VN problem considered here.
One formulation of prime interest in applications (especially in connectomics and social networks) is as follows:
given a collection of vertices of interest in one graph,
find those that play a similar structural (based on the topology of the underlying network) or functional (based on vertex or edge covariates) role in the other graph.
In addition, as seen in Section \ref{sec:features} the impact on VN consistency (and the potential existence of universally consistent schemes) when incorporating edge and vertex covariates into the VN framework is of prime interest, and a deeper analysis of the VN inference task in this framework is the subject of our current and future work.

The loss function considered in the present work is an analogue of the 0/1 recall-at-$k$ loss function in the information retrieval literature.  
Under this loss function, we have shown that no universally consistent VN rule exists.
It is natural to ask whether alternative loss functions can be considered under which universal consistency is achievable.  
While we conjecture that Example \ref{ex:behaviorchange} will nearly always provide a counterexample to universal consistency, this question remains open and is the subject of current research.


%
\section{Acknowledgments}
This work is supported in part by the D3M program of the Defense Advanced Research Projects Agency (DARPA), NSF grant DMS-1646108, and by the Air Force Research Laboratory and DARPA, under agreement number FA8750-18-2-0035. The U.S. Government is authorized to reproduce and distribute reprints for Governmental purposes notwithstanding any copyright notation thereon. The views and conclusions contained herein are those of the authors and should not be interpreted as necessarily representing the official policies or endorsements, either expressed or implied, of the Air Force Research Laboratory and DARPA, or the U.S. Government.

\appendix
\section{Proofs of main results}
Here we collect the proofs of the two main theorems in this work, Theorems \ref{lem:crossgraph:optimal} and \ref{thm:kconsis}.

\subsection{Theorem \ref{lem:crossgraph:optimal} and Remark \ref{rem:symm}}
\label{sec:apx:pfoptimal}
In this section, we present proofs of Theorem \ref{lem:crossgraph:optimal} and the claim in Remark \ref{rem:symm}.
\begin{proof}[Proof of Theorem \ref{lem:crossgraph:optimal}]
Recall the definition
\begin{equation*} 
\begin{aligned}
&(g_1,[\mo(g_2)])_{w=\mo(v)}
=\Big\{\big(g_1, \tilde g_2\big)\in\gn\times\gm\text{ s.t. }\exists\text{ iso. }\sigma\text{ with }\mo(\tilde g_2)=\sigma(\mo(g_2)),\\
&\hspace{55mm}\text{ and }\sigma(w)=\mo(v^*)\Big\}.
\end{aligned}
\end{equation*}
With $\bf{g}$ defined as in the theorem, note that for each $i\in[k]$,
\begin{align*}
U_{i,\mathbf{g}}^j :&= \left\{(g_1,g_2)\in \left(g_1^{(i)},[\mo(g_2^{(i)})]\right)\text{ s.t. }\rank_{\Phi(g_1,\mo(g_2),\vstar)}(\mo(\vstar)) =j\right\}\\
&=\left\{(g_1,g_2)\in \left(g_1^{(i)},[\mo(g_2^{(i)})]\right)\text{ s.t. }{\Phi(g_1,\mo(g_2),\vstar)}[j]=\mo(\vstar) \right\}\\
&=\Big\{(g_1,g_2)\in \left(g_1^{(i)},[\mo(g_2^{(i)})]\right)\text{ s.t. }\exists\text{ iso. }\sigma\text{ s.t. }\sigma(\mo(g_2^{(i)}))=\mo(g_2)\text{ and }\\
&\hspace{35mm}\sigma(\Phi(g_1^{(i)},\mo(g_2^{(i)}),\vstar)[j])= \mo(v^*)\Big\}\\
&=\left(g_1^{(i)},[\mo(g_2^{(i)})]\right)_{\Phi(g_1^{(i)},\mo(g_2^{(i)}),\vstar)[j]=\mo(v^*)}.
\end{align*}

For each $i\in[k]$ define $p_{i,\Phi}\in[0,1]^m$ via
\begin{equation*} 
p_{i,\Phi}[j,g_1^{(i)},g_2^{(i)}] =p_{i,\Phi}[j] 
:=\PFt\left(U_{i,\bf{g}}^j\,\big|\,(g_1^{(i)},[\mo(g_2^{(i)})])\right).
\end{equation*}
Observe that for each $i\in[k]$, it is immediate that $p_{i,\Phi^*}$ majorizes $p_{i,\Phi}$.
To see this, note that for any fixed $h$, letting $q^h_{i,\Phi}$ be $(p_{i,\Phi}[j])_{j=1}^h$ with entries sorted in descending order, we have
$p_{i,\Phi^*}[j]\geq q_{i,\Phi}^h[j]$ for all $j\in[h]$,
and majorization follows immediately.
With $\mathcal{P}_{\bf{g}}$ denoting the partition induced by $\bf{g}$, this majorization property implies
\begin{align*}
L_h(\Phi,v^*)&=1-\sum_{j=1}^h\p\left(\rank_{\Phi(G_1,\mo(G_2),\vstar)}(\mo(\vstar)) =j\right)\\
&=1-\sum_{\mathcal{P}_{\bf{g}}}\sum_{j=1}^h \p\left[ U_{i,\bf{g}}^j\,\big|\,(g_1^{(i)},[\mo(g_2^{(i)})]) \right] \p\left[ (g_1^{(i)},[\mo(g_2^{(i)})])\right]  \\
&=\sum_{\mathcal{P}_{\bf{g}}}\left(1-\sum_{j=1}^h p_{i,\Phi}[j,g_1^{(i)},\mo(g_2^{(i)})]\right)\p\left[ (g_1^{(i)},[\mo(g_2^{(i)})])\right]\\
&\geq \sum_{\mathcal{P}_{\bf{g}}}\left(1-\sum_{j=1}^h p_{i,\Phi^*}[j,g_1^{(i)},\mo(g_2^{(i)})]\right)\p\left[ (g_1^{(i)},[\mo(g_2^{(i)})])\right]\\
&=L_h(\Phi^*,v^*).
\end{align*}
As $\Phi$, $\bf g$, and $\mo$ were arbitrary, the proof follows.
\end{proof}

\begin{proof}[Proof of Remark \ref{rem:symm}]
Fix $i$, and let $\xi_i=|\Psi(g_1^{(i)},g_i^{(i)},v^*)|$.
Note that for each $j\leq \xi_i$, the set of graphs $(g_1,g_2)\in \left(g_1^{(i)},[\mo(g_2^{(i)})]\right)$ for which $\Psi(g_1,g_2,v^*)[j]=\mathcal{I}(\mo(v^*);\mo(g_2))$ is precisely 
$$\left(g_1^{(i)},[\mo(g_2^{(i)})]\right)_{\Psi(g_1^{(i)},\mo(g_2^{(i)}),\vstar)[j]=\mo(v^*)}.$$
If the tie breaking scheme $T$ satisfies $T[1]=\mo(v^*)$, then the set 
set of graphs $(g_1,g_2)\in \left(g_1^{(i)},[\mo(g_2^{(i)})]\right)$ for which $\Phi^*_T(g_1,g_2,v^*)[j]=\mo(v^*)$ is then also
$$\left(g_1^{(i)},[\mo(g_2^{(i)})]\right)_{\Psi(g_1^{(i)},\mo(g_2^{(i)}),\vstar)[j]=\mo(v^*)}.$$

The proof proceeds as follows.
For an arbitrary VN scheme $\Phi$, and for each $i\in[k]$,
\begin{align*}
U_{i,\mathbf{g}}^j :&= \left\{(g_1,g_2)\in \left(g_1^{(i)},[\mo(g_2^{(i)})]\right)\text{ s.t. }\rank_{\Phi(g_1,\mo(g_2),\vstar)}(\mo(\vstar)) =j\right\}\\
&\subset\Big\{(g_1,g_2)\in \left(g_1^{(i)},[\mo(g_2^{(i)})]\right)\text{ s.t. }\exists\text{ iso. }\sigma\text{ s.t. }\sigma(\mo(g_2^{(i)}))=\mo(g_2)\text{ and }\\
&\hspace{35mm}\sigma\left(\mathcal{I}\left(\Phi(g_1^{(i)},\mo(g_2^{(i)}),\vstar)[j];\mo(g_2^{(i)})\right)\right)\supset \mo(v^*)\Big\}\\
&=\left(g_1^{(i)},[\mo(g_2^{(i)})]\right)_{\mathcal{I}\left(\Phi(g_1^{(i)},\mo(g_2^{(i)}),\vstar)[j];\mo(g_2^{(i)})\right)=\mo(v^*)}
\end{align*}
Letting $T[1]=\mo(v^*)$, this implies that 
\begin{align*}
&1-L_\ell(\Phi,v^*)=\sum_{\mathcal{P}_{\mathbf{g}}}\p\left[\left(\bigcup_{j=1}^\ell U_{i,\mathbf{g}}^j\right)\,\big|\, \left(g_1^{(i)},[\mo(g_2^{(i)})]\right)\right]\p\left[\left(g_1^{(i)},[\mo(g_2^{(i)})]\right)\right]\\
&\leq \sum_{\mathcal{P}_{\mathbf{g}}}\p\left[\bigcup_{j=1}^\ell \left(g_1^{(i)},[\mo(g_2^{(i)})]\right)_{\mathcal{I}\left(\Phi(g_1^{(i)},\mo(g_2^{(i)}),\vstar)[j];\mo(g_2^{(i)})\right)=o(v^*)}\,\big|\,\left(g_1^{(i)},[\mo(g_2^{(i)})]\right)\right]\p\left[\left(g_1^{(i)},[\mo(g_2^{(i)})]\right)\right]\\
&=\sum_{\mathcal{P}_{\mathbf{g}}}\sum_{j\in\mathcal{J}_i} \p\left[\left(g_1^{(i)},[\mo(g_2^{(i)})]\right)_{\mathcal{I}\left(\Phi(g_1^{(i)},\mo(g_2^{(i)}),\vstar)[j];\mo(g_2^{(i)})\right)=o(v^*)} \, |\, \left(g_1^{(i)},[\mo(g_2^{(i)})]\right)\right]\p\left[\left(g_1^{(i)},[\mo(g_2^{(i)})]\right)\right]\\
&\leq \sum_{\mathcal{P}_{\mathbf{g}}}\sum_{j\in\mathcal{J}_i} \underbrace{\p\left[\left(g_1^{(i)},[\mo(g_2^{(i)})]\right)_{\Psi(g_1^{(i)},\mo(g_2^{(i)}),\vstar)[j]=\mo(v^*)} \, |\, \left(g_1^{(i)},[\mo(g_2^{(i)})]\right)\right]}_{ =
\p\left[(g_1,g_2)\in\left(g_1^{(i)},[\mo(g_2^{(i)})]\right)\text{ s.t. }\BayesVN_T(g_1,\mo(g_2),v^*)[j]=\mo(v^*)\,\big|\, \left(g_1^{(i)},[\mo(g_2^{(i)})]\right)\textbf{}\right]
}\p\left[\left(g_1^{(i)},[\mo(g_2^{(i)})]\right)\right]\\
&\leq 1-L_{\ell}(\BayesVN_T,v^*),
\end{align*}
where $\mathcal{J}_i\subset[\ell]$ is the lexicographically smallest set of indices for which 
$$\left\{ \left(g_1^{(i)},\mo(g_2^{(i)})]\right)_{\mathcal{I}\left(\Phi(g_1^{(i)},\mo(g_2^{(i)}),\vstar)[j];\mo(g_2^{(i)})
\right)=o(v^*)} \right\}_{j\in\mathcal{J}_i}$$ are distinct.
\end{proof}

\subsection{Proof of Theorem \ref{thm:kconsis}}
\label{sec:apx:kconsis}

\begin{proof}
Define a probability vector $\xi \in \R^m$ by
$\xi_i = \epsilon_{i}-\epsilon_{i-1}$ for $i \in [m-1]$
(where we take $\epsilon_0:=0$), and let $\xi_m=1-\epsilon_{m-1}$.
Consider asymmetric graphs $(g_1, g_2)\in\mathcal{G}_n\times\gm$
and construct a distribution $F_{c,n,m,\theta}\in\mathcal{N}$ as follows.
\begin{itemize}
        \item[i.] $c=n\wedge m$;
        \item[ii.] The support of $F_{c,n,m,\theta}$ is $(g_1,[\mo(g_2)])$;
        \item[iii.] For each $k\in[m]$ define
\begin{equation*} 
R_{\Phi,k}=\big\{(g_1,\tilde g_2)\in(g_1,[\mo(g_2)])\text{ s.t. }\Phi(g_1,\mo(\tilde g_2),v^*)[k]=\mo(\vstar)\big\}.
\end{equation*}
Then we define $\p_{F_{c,n,m,\theta}}(R_{\Phi,k}):=\xi[k]$ with all elements of $R_{\Phi,k}$ being assigned equal mass under $F_{c,n,m,\theta}$.
\end{itemize}
It is clear then that $L_k(\Phi,v^*)=1-\epsilon_k>1-\frac{k}{m}$.
It is also clear that $\Lstar_k(v^*)\leq \epsilon_{m-k}$.
Indeed, consider $\Phi'$ which is defined by reversing the order provided by $\Phi$; then $L_k(\Phi',v^*)=\epsilon_{m-k}$; which completes the proof.
\end{proof}

\section{Proof of Theorem \ref{thm:stochmatchedcores}}
\label{sec:gmpf}
Herein we will provide a sketch of the proof of Theorem \ref{thm:stochmatchedcores} for completeness.
Let $Q$ be a permutation matrix in $\Pi_n$ that permutes precisely $k$ labels (i.e., $\sum_i Q_{i,i}=n-k$), and let $T$ denote the number of transpositions induced by $Q$.  
By exploiting the cyclic structure of $Q$ acting on $\text{vec}(B)$, we have that 
$$\mathbb{E}\delta(Q)-\mathbb{E}\delta(I_n)=\mathbb{E}\|AQ-QB\|_F^2-\mathbb{E}\|A-B\|_F^2\geq \epsilon\left((n-k)k+\binom{k}{2}-T\right).$$
Combining this expectation bound with the following McDiarmid-like concentration result will yield the proof of Theorem \ref{thm:stochmatchedcores}.
\begin{proposition}[Proposition 3.2 from \cite{kim}]\label{prop:kim}
Let $X_1,\dotsc,X_m$ be a sequence of independent Bernoulli random variables where $\mathbb{E}[X_i]=p_i$. 
Let $f:\{0,1\}^m\mapsto \Re$ be such that changing any $X_i$ to $1-X_i$ changes $f$ by at most 
\[ 
    M=\sup_i \sup_{X_1,\dotsc,X_n} |f(X_1,\dotsc,X_i,\dotsc,X_n) - f(X_1,\dotsc,1-X_i,\dotsc,X_n)|.
\]
Let $\sigma^2 = M^2 \sum_i p_i(1-p_i)$ and let $Y=f(X_1,\dotsc,X_n)$.

Then $$\Pr[|Y-\mathbb{E}[Y]| \geq t \sigma ] \leq 2 e^{-t^2/4}$$ for all $0<t<2\sigma/M$.
\end{proposition}
Indeed, we see that $X_Q:=\delta(Q)-\delta(I_n)$ is a function of $N_Q$ independent Bernoulli random variables, where
$
N_Q =3\left(\binom{k}{2}+k(n-k)\right)\leq 3kn.
$
Let $S_P$ be the sum of these $N_P$ Bernoulli random variables, 
and it follows that $\text{Var}(S_P)\leq N_P/4$.
By setting $t=C\frac{\epsilon kn}{\sigma}$ for an appropriate constant $C>0$ in Proposition \ref{prop:kim}, we have
\begin{align*}
\Pr\left(X_Q\leq 0\right)&\leq\Pr\left(|X_Q-\mathbb{E}(X_Q)|\geq \mathbb{E}(X_Q)\right)\leq 2\text{exp}\left\{ -\Theta(\epsilon^2 kn) \right\}.
\end{align*}
A union bound over all such $Q$ (of which there are $\leq n^{k}$) and over $k$
yields
\begin{align*}
\mathbb{P}\left(\exists\,\,Q\in \Pi_n\setminus\{I_n\}\text{ s.t. } \delta(Q)\leq \delta(I_n)\right)
\leq& \sum_{k=2}^{n}  2 \text{exp}\left\{k\log(n)-\Theta(\epsilon^2 kn)  \right\}= 2\text{exp} \left\{ -\Theta(\epsilon^2 n) \right\},
\end{align*}
as desired.

\subsection{VN schemes with ties}
\label{sec:VNties}

We can incorporate ties into the VN framework as follows.
With ties allowed,
any sensibly-defined vertex nomination scheme should view all vertices in $\mathcal{I}(u;g)$ as being equally good matches to a vertex of interest $v^*$.
To this end, we will view VN schemes as providing \emph{weak orderings} of the elements of $W$:
\begin{definition}
For a set $A$, let $\mathfrak{W}_A$ denote the set of all weak orderings of the elements of $A$ (i.e., the set of all total orderings where ties are allowed).
For $x\in\mathfrak{W}_A$, let $t_x$ be any maximum-length total ordering induced by $x$.
For each $a\in A$, we define 
$$\text{rank}_x(a)=\text{rank}_{t(x)}(a'),$$
where $a=a'$ according to the ordering $x$.
\end{definition}
\begin{example}
If $A=\{a,b,c,d,e\}$ and $x: a>c>d=e>b$, then $t(x): a>c>d>b$, or $t(x): a>c>e>b$; in either case, $\text{rank}_x(a)=1$, $\text{rank}_x(c)=2$, $\text{rank}_x(d)=3$, $\text{rank}_x(e)=3$, and $\text{rank}_x(b)=4$.
\end{example}

A well-defined VN scheme should be ``label-independent'' in the following sense:
Each element of each $\mathcal{I}(\mo(u);\mo(g_2))$ should be ranked identically by $\Phi$, and
these ranks should be independent of the obfuscation function $\mo$.
Formally, we have the following.

\begin{definition}[Vertex Nomination (VN) Scheme]\label{def:VNties}
Let $\mathfrak{O}_W$ be the set of all obfuscating functions $\mo:V_2\mapsto W$ for a fixed $W$.
For $n,m>0$ fixed, a \emph{vertex nomination scheme} is a function $\Phi: \gn \times \gm \times \mathfrak{O}_W \times V_1 \rightarrow \mathfrak{W}_{W}$ satisfying the following properties:  For all $(g_1,g_2)\in\gn\times\gm$,
\begin{itemize}
\item[i.] If $u_1\notin\mathcal{I}(u_2;g_2)$ then either
$\mo(u_1)>\mo(u_2)$ or $\mo(u_1)<\mo(u_2)$ in the ordering provided by $\Phi(g_1,\mo(g_2),v^*)$;
\item[ii.] If $u_1\in\mathcal{I}(u_2;g_2)$ then $\mo(u_1)=\mo(u_2)$ in the ordering provided by $\Phi(g_1,\mo(g_2),v^*)$;
\item[iii.] (consistency criterion) If $\mo_1,\mo_2\in\mathfrak{O}_W$, then for each $v\in V(g_2)$
\begin{align}
\label{eq:consisties}
\text{rank}_{\Phi(g_1,\mo_1(g_2),v^*)}(\mo_1(v))=\text{rank}_{\Phi(g_1,\mo_2(g_2),v^*)}(\mo_2(v)).
\end{align}
\end{itemize}
%
We let $\calVnm$ denote the set of all such VN schemes.
\end{definition}

The VN loss functions and level-$k$ errors are defined as in the totally ordered setting.
To define the Bayes optimal scheme, let $(g_1,g_2)$ be realized from $(G_1,G_2)\sim F_{c,n,m,\theta}\in\mathcal{N}$, and consider a vertex of interest $v^*\in C$
and obfuscating function $\mo:V_2\rightarrow W$.
Letting $\simeq$ denote graph isomorphism, define the set
\begin{equation} 
\begin{aligned}
(g_1,[\mo(g_2)])
&= \left\{\big(g_1, \tilde g_2\big)\in\gn\times\gm\text{ s.t. }\mo(\tilde g_2)\simeq\mo(g_2)\right\} \\
&= \left\{\big(g_1, \tilde g_2\big)\in\gn\times\gm\text{ s.t. }\tilde g_2\simeq g_2\right\}
\end{aligned}
\end{equation}
In order to define the Bayes optimal scheme, we will also need the following restrictions of $(g_1,[\mo(g_2)])$: for each $w\in W,$ and $v\in V_2$ we define
\begin{equation} 
\begin{aligned}
&(g_1,[\mo(g_2)])_{\mathcal{I}(w;\mo(g_2))=\mo(v)}
=\Big\{\big(g_1, \tilde g_2\big)\in\gn\times\gm\text{ s.t. }\exists\text{ iso. }\sigma\text{ with }\mo(\tilde g_2)=\sigma(\mo(g_2)),\\
&\hspace{55mm}\text{ and }\sigma(u)=\mo(v)\text{ for some }u\in\mathcal{I}(w;\mo(g_2))\Big\}.
\end{aligned}
\end{equation}
Note that for a fixed $v$, if $\{\mathcal{I}(w;\mo(g_2))\}_{w\in W'}$ partitions $W$ (for some suitable $W' \subseteq W$), then 
$$\left\{(g_1,[\mo(g_2)])_{\mathcal{I}(w;\mo(g_2))=\mo(v)}\right\}_{w\in W'}$$ partitions $(g_1,[\mo(g_2)])$.
We are now ready to define a Bayes optimal VN scheme.

For ease of notation, in the sequel we will write $\PFt$ or even simply $\p$
in place of $\PFtau$ where there is no risk of ambiguity.
Let 
\begin{equation}
\label{eq:g2}
\mathbf{g}=\left\{\left(g_1^{(i)},g_2^{(i)}\right)\right\}_{i=1}^k
\end{equation} 
be such that the sets 
$$\left\{\left(g_1^{(i)},[\mo(g_2^{(i)})]\right)\right\}_{i=1}^k
$$
 partition $\gn\times\gm$.
We will call this partition $\mathcal{P}_{n,m}$, where we suppress dependence on
$\mathbf{g}$ and $\mo$ for ease of notation.
We will define a Bayes optimal scheme $\BayesVN=\BayesVN_{\bf{g}}$ piecewise on each element of this partition.
For each $i\in[k]$, define (where ties in the argmax's are broken in an arbitrary but nonrandom manner)
\begin{equation} \label{eq:bot} 
\begin{aligned}
&\BayesVN \!\!\left(g_1^{(i)},\mo(g_2^{(i)}),v^*\right)\![1]\! \in \underset{\mathcal{I}(u;\mo(g_2^{(i)})) \subset  W}{\am}\,\,
\p\! \left[ \left(g_1^{(i)},[\mo(g_2^{(i)})]\right)_{\mathcal{I}(u;\mo(g_2^{(i)}))=\mo(v^*)} \bigg| \left(g_1^{(i)},[\mo(g_2^{(i)})]\right) \right] \\
&\BayesVN \!\!\left(g_1^{(i)},\mo(g_2^{(i)}),v^*\right)\![2]\!
\in \!\!\!\!\!\!\!\!\underset{\mathcal{I}(u;\mo(g_2^{(i)})) \subset W\setminus \{ \BayesVN[1] \} }{\am}
\p\! \left[ \left(g_1^{(i)},[\mo(g_2^{(i)})]\right)_{\mathcal{I}(u;\mo(g_2^{(i)}))=\mo(v^*)} \bigg| \left(g_1^{(i)},[\mo(g_2^{(i)})]\right) \right] \\
&\hspace{50mm}\vdots \\
&\BayesVN \!\!\left(g_1^{(i)},\mo(g_2^{(i)}),v^*\right)\![k]\!
\in \!\!\!\!\!\!\!\!\underset{\mathcal{I}_{\mo\left(g_2\right)}(u) \subset W\setminus\{\cup_{i=1}^{k-1}\Phi^*[i]\} }{\am}
\p\! \left[ \left(g_1^{(i)},[\mo(g_2^{(i)})]\right)_{\mathcal{I}(u;\mo(g_2^{(i)}))=\mo(v^*)} \bigg| \left(g_1^{(i)},[\mo(g_2^{(i)})]\right) \right]
\end{aligned}
\end{equation}
so that the ranking provided by $\Phi^*$ is 
(where $\Phi^{*}(g_1^{(i)},\mo(g_2^{(i)}),v^*)[i]=\{u_1^{(i)},\ldots,u_{n_i}^{(i)}\}$)
$$\underbrace{u_1^{(1)}=u_2^{(1)}=\cdots=u_{n_1}^{(1)}}_{n_1}> 
\underbrace{u_1^{(2)}=u_2^{(2)}=\cdots=u_{n_2}^{(2)}}_{n_2}>
\cdots>
\underbrace{u_1^{(k)}=u_2^{(k)}=\cdots=u_{n_k}^{(k)}}_{n_k}.  
$$
For each element 
$$(g_1',g_2')\in( g^{(i)}_1,[\mo(\tilde g_2^{(i)})])\setminus\{(g_1^{(i)},g_2^{(i)})\},$$ choose any isomorphism $\sigma$ such that $\mo(g_2')=\sigma(\mo(g_2^{(i)}))$, and define
$$\BayesVN(g_1,\mo(g_2'),v^*)=\sigma(\BayesVN(g_1^{(i)},\mo(g_2^{(i)}),v^*)),$$
noting that $\BayesVN(g_1,\mo(g_2'),v^*)$ is independent of the choice of isomorphism $\sigma$.
The next proposition states that, modulo ties, the definition of $\Phi^*_{\bf{g}}$ is independent of the choice of $\bf{g}$.
\begin{proposition}
\label{prop:ind}
Let $\mathfrak{o}\in\mathfrak{O}_W$ be an obfuscating function, and let
$$\mathbf{g}=\left\{\left(g_1^{(i)},g_2^{(i)}\right)\right\}_{i=1}^k\neq \tilde{\mathbf{g}}=\left\{\left( g_1^{(i)},\tilde g_2^{(i)}\right)\right\}_{i=1}^k$$ be such that the sets 
$$\left\{\left(g_1^{(i)},[\mo(g_2^{(i)})]\right)\right\}_{i=1}^k, \left\{\left(g_1^{(i)},[\mo(\tilde g_2^{(i)})]\right)\right\}_{i=1}^k$$
 partition $\gn\times\gm$.
 Suppose that $(G_1,G_2)\sim F_{c,n,m,\theta}\in\mathcal{N}$, and consider a vertex of interest $v^*\in C$.
 Then there exists a fixed strategy for breaking ties in the argmax's for $\Phi_{\bf{g}}^*$ and $\Phi_{\tilde{\bf{g}}}^*$ that yields $\Phi_{\bf{g}}^*=\Phi_{\tilde{\bf{g}}}^*$.
 In particular, under any such tie-breaking strategy, we have that $L_h(\BayesVN_{\bf{g}},v^*)=L_h(\BayesVN_{\tilde{\bf{g}}},v^*)$ for all $h\in[m-1]$.
\end{proposition}

Lastly, the following theorem shows that this scheme (or schemes) is indeed Bayes optimal.
The proof is analogous to the totally ordered setting and is thus omitted.

\begin{theorem} 
\label{lem:crossgraph:optimalties}
Let $\mathfrak{o}\in\mathfrak{O}_W$ be an obfuscating function, and let
$$\mathbf{g}=\left\{\left(g_1^{(i)},g_2^{(i)}\right)\right\}_{i=1}^k$$ be such that the sets 
$$\left\{\left(g_1^{(i)},[\mo(g_2^{(i)})]\right)\right\}_{i=1}^k$$
 partition $\gn\times\gm$.
Let $\BayesVN=\BayesVN_\mathbf{g}$ be as defined in Equation~\eqref{eq:bo}.
Suppose that $(G_1,G_2)\sim F_{c,n,m,\theta}\in\mathcal{N}$, and consider a vertex of interest $v^*\in C$.
We have that $L_h(\BayesVN,v^*)=\Lstar_h(\vstar)$ for all $h\in[m-1]$ and all obfuscating functions $\mathfrak{o}$.
\end{theorem}



%

%

\bibliographystyle{plain}
\bibliography{biblio.bib}
\end{document}